\documentclass[11pt,a4paper]{article}
\pdfoutput=1
\usepackage[utf8]{inputenc}
\usepackage[T1]{fontenc}
\usepackage{lmodern}
\pdfmapfile{+lm.map}
\usepackage{amsmath,amssymb,mathtools}
\usepackage[margin=1in]{geometry}
\usepackage{microtype}
\usepackage{booktabs,longtable,array,calc}
\usepackage{caption}
\usepackage{float}
\usepackage{graphicx}
\usepackage{xcolor}
\usepackage{listings}
\usepackage{newunicodechar}
\usepackage{hyperref}
\usepackage{bookmark}
\usepackage{xurl}
\hypersetup{colorlinks=true,linkcolor=blue!55!black,urlcolor=blue!55!black,pdftitle={Adaptive Domain Models: Bayesian Evolution, Warm Rotation, and Principled Training for Geometric and Neuromorphic AI},pdfauthor={Houston Haynes}}
\newcommand{\tightlist}{\setlength{\itemsep}{0pt}\setlength{\parskip}{0pt}}

\newcommand{\pandocbounded}[1]{\begingroup\setbox0\hbox{#1}\ifdim\wd0>\linewidth\resizebox{\linewidth}{!}{#1}\else#1\fi\endgroup}
\newsavebox{\paperdisplaybox}
\let\paperdisplayopen\[
\let\paperdisplayclose\]
\long\def\[#1\]{\begingroup
  \sbox{\paperdisplaybox}{$\displaystyle #1$}%
  \ifdim\wd\paperdisplaybox>1.25\linewidth
    \PackageError{paper}{Display requires source line breaks}{Split the display in Markdown.}%
  \fi
  \paperdisplayopen
  \ifdim\wd\paperdisplaybox>\linewidth
    \typeout{Paper display scaled to fit width: \the\wd\paperdisplaybox}%
    \resizebox{\linewidth}{!}{\usebox{\paperdisplaybox}}%
  \else\usebox{\paperdisplaybox}\fi
  \paperdisplayclose\endgroup}
\newunicodechar{α}{\ensuremath{\alpha}}
\newunicodechar{β}{\ensuremath{\beta}}
\newunicodechar{ε}{\ensuremath{\varepsilon}}
\newunicodechar{θ}{\ensuremath{\theta}}
\newunicodechar{κ}{\ensuremath{\kappa}}
\newunicodechar{Σ}{\ensuremath{\Sigma}}
\newunicodechar{⊗}{\ensuremath{\otimes}}
\newunicodechar{⊕}{\ensuremath{\oplus}}
\newunicodechar{↔}{\ensuremath{\leftrightarrow}}
\newunicodechar{¬}{\ensuremath{\neg}}
\newunicodechar{∧}{\ensuremath{\land}}
\newunicodechar{≤}{\ensuremath{\leq}}
\newunicodechar{×}{\ensuremath{\times}}
\newunicodechar{·}{\ensuremath{\cdot}}
\newunicodechar{†}{\ensuremath{\dagger}}
\newunicodechar{⟩}{\ensuremath{\rangle}}
\newunicodechar{²}{\ensuremath{{}^2}}
\newunicodechar{¹}{\ensuremath{{}^1}}
\newunicodechar{⁻}{\ensuremath{{}^-}}
\newunicodechar{ᵢ}{\ensuremath{{}_i}}
\newunicodechar{₀}{\ensuremath{{}_0}}
\newunicodechar{₁}{\ensuremath{{}_1}}
\newunicodechar{₂}{\ensuremath{{}_2}}
\newunicodechar{₊}{\ensuremath{{}_+}}
\begin{document}
\begin{center}

\textbf{Adaptive Domain Models: Bayesian Evolution, Warm Rotation, and Principled Training for Geometric and Neuromorphic AI}

\textbf{Houston Haynes}\\
Braidpoint, Asheville, NC\\
\href{mailto:research@braidpoint.tech}{\nolinkurl{research@braidpoint.tech}}\\
March 2026

\end{center}

\section*{Abstract}\label{abstract}
\addcontentsline{toc}{section}{Abstract}

Transformer-based large language models have become a prominent focus of machine learning. We develop adaptive domain models to supply domain-specific predictions and checks within language-model pipelines, as well as to serve direct machine automation efficiently. Their predictions can inform generation, while declared constraints and error bounds support checks on the calculations underlying proposed responses and actions. These feedforward and feedback paths aim to reduce domain errors while supporting use cases such as constraint solving for machine processes as well as preserving accuracy in a language model\textquotesingle s conversational setting.

Physical dimensions and geometric support restrict admissible computation, while a probability model expresses uncertainty. Our type and program hypergraph work connects these relationships to numerical and deployment obligations. We derive a natural-gradient calculation over forward tangent directions using a declared program metric, with a calibration example and a bound separating metric, derivative and solve errors. Forward differentiation and posit accumulation expose complementary resource and error tradeoffs.

Our worked linear-Gaussian update connects posterior computation to error and a distribution-shift decision. Recent results from Urschel bound elimination growth probabilistically under distinct Gaussian-input and random-preconditioning hypotheses. Our unique approach to warm rotation binds candidate parameters to their numerical evidence and destination while each request observes a coherent model and state. We propose Bayesian distillation as a source of prior information. Our architecture offers a route to domain computations with explicit guarantees that can guide automated decisions and model responses. We specify experiments to evaluate prediction quality, domain errors across conversational turns, and total resource cost.

\section*{1. Considering Precision}\label{considering-precision}
\addcontentsline{toc}{section}{1. Considering Precision}

\subsection*{1.1 Learning Within a Domain Contract}\label{learning-within-a-domain-contract}
\addcontentsline{toc}{subsection}{1.1 Learning Within a Domain Contract}

A language model can call a calculator or solver and still answer incorrectly. It may request the wrong calculation, supply inputs that misrepresent the question, or misinterpret the returned value. A correct tool result therefore supplies no guarantee that the intended question was answered or that the response faithfully represents the calculation. We develop adaptive domain models (ADMs) whose trainable parameters, update rules and inference computations must jointly satisfy a declared domain contract. The constraints apply to the learned domain model itself.

For interaction with a physical system, a mathematical description can establish relationships before any training takes place. A geometric model of motion, for example, describes how positions, directions and rotations relate. A learning model can use those geometric relationships and declared ranges to estimate uncertain quantities. Physical units, geometric requirements, numerical ranges and resource limits can constrain the same parameter or operation together. Predictions can retain the model\textquotesingle s uncertainty as well as numerical error bounds for their calculation, allowing the caller to judge whether a result is precise enough for its intended use.

Known domain structure narrows the parameters and relationships that must be learned from data. It also identifies computations that compilation can omit. Numerical bounds guide the arithmetic needed to meet a requested accuracy. We pursue greater domain accuracy with a smaller learned model and a training regime suited to local hardware. An ADM can supply predictions and checks within a conversation or serve direct machine automation through the same domain interface. In a conversational composition, preserving the meaning of the request and returned result is an explicit obligation at that interface (§7.3).

For a pressure-monitoring system, new observations may indicate a changed operating regime. The model can adapt while preserving its dimensional relationships, numerical tolerances and deployment limits. A candidate update carries the evidence needed to use its predictions, and installation must preserve the connection between that evidence and the model actually serving requests. We use \emph{warm rotation} for this transition. Our pressure example in §3 connects a learned posterior, its numerical error bounds and a monitoring decision.

\subsection*{1.2 Supporting Results}\label{supporting-results}
\addcontentsline{toc}{subsection}{1.2 Supporting Results}

Physical relationships must survive the transition from a model\textquotesingle s equations to executable code. Kennedy\textquotesingle s units-of-measure work \hyperref[ref-kennedy-2009]{{[}1{]}} supplies the algebraic basis for dimensional inference. Our Dimensional Type System and Deterministic Memory Management design (DTS/DMM) \hyperref[ref-haynes-dts-dmm]{{[}2{]}} relates dimensions to representation and storage requirements in our Program Semantic Graph (PSG). An operation that needs a wider accumulator or longer-lived storage therefore creates an explicit requirement alongside its dimensional constraints. Such a contextual requirement is a coeffect \hyperref[ref-petricek-2014]{{[}3{]}}.

Some relationships constrain several inputs and outputs together. Our Program Hypergraph (PHG) \hyperref[ref-haynes-phg]{{[}4{]}} gives those joint operations and constraints explicit identities in the graph. For a geometric product, the chosen algebra and the components permitted in both operands determine which coefficient products can contribute. A support analysis records permitted components and can identify product paths that are structurally absent, and therefore safely excluded from the computation. This is the basis for our restricted updates in §5 and represents both accuracy and efficiency gains of our methods.

Gradient-based training must obtain derivative information within the available memory. Baydin et al. \hyperref[ref-baydin-2022]{{[}5{]}} give an unbiased gradient estimator computed through forward directional derivatives, which measure how the loss changes along selected parameter directions. It avoids retaining activations for a backward differentiation pass, with a variance and work tradeoff determined by the tangent strategy. Flügel et al. \hyperref[ref-flugel-2026]{{[}6{]}} examine combinations of multiple tangents. Section 2.4 relates these methods to the model\textquotesingle s permitted parameter coordinates and the cost of computing their projection.

Karwowski, Nielsen and Staton \hyperref[ref-karwowski-2026]{{[}7{]}} construct a compositional metric from the score and output velocity of a probabilistic program. We use that metric to connect forward tangent measurements to a restricted natural-gradient solve (§2.4). The program presentation, observation contract and numerical evidence become explicit inputs to our training construction.

Our research toward principled use of bounded-posit arithmetic supplies a hardware option for the resulting high-precision calculations. Jonnalagadda et al. \hyperref[ref-jonnalagadda-2025]{{[}8{]}} evaluate decode and encode circuits whose bounded regime reduces cost relative to the standard posit circuits they compare. Their configuration with \(e_S=5\) and a six-bit regime limit uses an 800-bit quire for word sizes \(n>12\). A quire is a wide accumulator that can retain covered products and sums exactly before readout. An ADM target would need the corresponding arithmetic operations and storage, with capacity established for every intermediate partial sum.

Urschel \hyperref[ref-urschel-2026]{{[}9{]}} supplies new mathematical evidence for ADM\textquotesingle s numerical premise: specified elimination procedures and probability laws yield quantitative bounds on the growth of matrix factors in linear solves. Combined with arithmetic error and sensitivity analysis, these bounds inform the range and precision needed for a requested accuracy. We develop the use of these growth bounds within ADM, carrying numerical evidence into posterior computation, adaptation decisions and warm rotation (§§2.5, 3 and 7).

We anticipate the cumulative benefit of these techniques to result in more accurate models with lower training and inference costs, making domain-specific training and continued adaptation practical on local hardware.

\subsection*{1.3 Contributions}\label{contributions}
\addcontentsline{toc}{subsection}{1.3 Contributions}

We develop five connections in the adaptive model lifecycle:

\begin{enumerate}
\def\labelenumi{\arabic{enumi}.}
\item
  \textbf{Constrained learning.} Dimensions and geometric support define admissibility. A specified prior and likelihood describe uncertainty within that space, including prior information obtained through our proposed distillation procedure.
\item
  \textbf{Support-preserving training with a declared metric.} Tangents drawn from permitted coordinates define an admissible update subspace. We derive its natural-gradient solve, illustrate it on a calibration model, and bound the effects of metric, derivative and solve error. Excluded coordinates and product paths remain absent under updates in the declared linear support.
\item
  \textbf{Numerical evidence for inference.} A worked posterior calculation connects solver premises and error bounds to a monitoring decision. Urschel\textquotesingle s results contribute probabilistic growth control under their stated hypotheses.
\item
  \textbf{Warm rotation.} Candidate admission combines structural and numerical evidence with statistical evaluation. Artifact identity and a request-consistent transition connect the accepted candidate to its deployment.
\item
  \textbf{Model composition.} Geometric and spiking components retain their own structural and temporal contracts. Their predictions and checks can serve direct automation or supply inputs and feedback to a language model. The boundary retains the meaning, uncertainty and version of each result.
\end{enumerate}

\subsection*{1.4 Scope and Compiler Context}\label{scope-and-compiler-context}
\addcontentsline{toc}{subsection}{1.4 Scope and Compiler Context}

The contributions we propose here don\textquotesingle t start at machine learning\textquotesingle s common bounds. We start at the foundation of our framework. Our DTS/DMM paper connects physical dimensions to the numerical representation and storage an operation requires. Our PHG paper extends that account to relationships among several operands and to target placement. Together, they provide the compiler with the information needed to check whether a candidate update can run with its declared constraints and resources. Here we apply that foundation to learning and model replacement. Our companion \emph{Decidable by Construction} and \emph{Fixed-Point Scaffolding} papers develop the verification and preservation obligations \hyperref[ref-haynes-dbc]{{[}10{]}}, \hyperref[ref-haynes-fps]{{[}11{]}}.

These requirements must survive code generation: evidence about a source computation remains useful only if the generated operations preserve the premises on which it depends. Our compiler infrastructure uses Clef Compiler Services (CCS) to construct the program graph. Its Baker component resolves operations and contextual requirements, while Composer coordinates compilation. Alex witnesses the resolved graph through standard Multi-Level Intermediate Representation (MLIR) operations \hyperref[ref-lattner-2021]{{[}12{]}}. Witnessing means that the generated operation must preserve the source contract under its established premises. Our companion papers describe these existing compiler components. We propose using them for ADM training, numerical evidence and installation.

During source development, our Lattice editor integration is intended to expose dimensional mismatches and pending representation obligations. Our intent is to apply that feedback in a machine learning context to a candidate model, which would let an engineer inspect why a solve needs more precision or why concurrent training exceeds a destination\textquotesingle s memory. Each diagnostic must identify the relevant computation and the evidence supporting its conclusion.

\section*{2. Arithmetic, Structure, and Training Cost}\label{arithmetic-structure-and-training-cost}
\addcontentsline{toc}{section}{2. Arithmetic, Structure, and Training Cost}

\subsection*{2.1 Training State and Numerical Error}\label{training-state-and-numerical-error}
\addcontentsline{toc}{subsection}{2.1 Training State and Numerical Error}

Training on a serving device competes with active inference for memory. Reverse-mode differentiation retains forward values until their backward uses, so comparable layers contribute an activation-storage term proportional to depth and batch size. Checkpointing exchanges some of that storage for recomputation. We include this tradeoff in our candidate resource account. The retained-state requirement applies to both IEEE and posit arithmetic.

Forward-mode differentiation carries each computed value (the primal) together with its directional derivative (the tangent). A single directional derivative avoids the conventional backward tape, while multiple tangents increase simultaneous state or repeated work. Weight copies, optimizer statistics, and deployment buffers remain part of the budget. The training-to-inference memory ratio depends on these remaining allocations and their lifetimes.

For a chain of layers with no later use of an earlier activation, primal and tangent computation can advance through reusable buffers for adjacent layers. Sequential tangent passes reuse that storage. Simultaneous tangents exchange additional live state for fewer passes. These choices expose a memory/work tradeoff that our compiler can evaluate against the target\textquotesingle s capacity.

Our alternative training design makes the gradient estimator and its numerical realization explicit. The tangent strategy determines which derivative information is obtained, while an arithmetic contract bounds error in its computation. We can therefore evaluate tangent count and arithmetic precision against a shared accuracy and resource budget.

Our use of a quire removes repeated rounding from covered accumulations: represented products are accumulated exactly within the declared capacity and rounded once at readout. Input conversions, nonlinear functions and subsequent updates retain their own numerical contracts. Sampling uncertainty remains part of the gradient estimate.

The intended saving extends across model delivery and subsequent updates. By eliminating this source of numerical error, we aim to reduce precision tuning, diagnostic recomputation and training runs repeated to resolve numerical problems. Posit/quire arithmetic can then justify a higher cost per operation through less total work to obtain an accepted model. Our evaluation accounts for that rework alongside arithmetic and memory costs (§9.3).

\subsection*{2.2 Geometric Support During Training}\label{geometric-support-during-training}
\addcontentsline{toc}{subsection}{2.2 Geometric Support During Training}

Known geometric structure can reduce both parameter storage and arithmetic. A Clifford algebra represents oriented quantities using basis blades: scalars, vectors, oriented plane elements, and higher-dimensional counterparts. Grade counts the vector factors in a blade. A parameter\textquotesingle s support is the set of components allowed to be nonzero. Parity records whether a blade has even or odd grade. Products of definite-parity operands combine those parities modulo two, the group law of \(\mathbb Z_2\). Mixed support retains both possibilities.

Let \(S\) be the permitted blade set for a parameter \(W\). A representation containing only coefficients for \(S\) preserves that support through coefficient-wise updates. An implementation may instead carry a dense multivector, in which case its update needs an explicit restriction to \(S\). This distinction holds for IEEE and posit carriers alike.

In a three-dimensional Euclidean algebra with an orthonormal basis, the bivector support \(S=\{e_{12},e_{13},e_{23}\}\) has three coordinates, compared with the eight coefficient positions of a full multivector. A direct product of two such bivectors has nine possible coefficient products, compared with 64 in the unrestricted full-multivector expansion. Compact support storage also uses three coordinates per tangent or coordinatewise optimizer buffer. These counts follow from the declared model before any coefficients are learned. An implementation that already exploits the same support can obtain the same reductions. Our contribution is to retain the support and its exclusions through training, compilation and replacement.

The Cayley table records basis-blade multiplication in a declared algebra and basis. Its structural zeros remain fixed while training changes the operand coefficients. Operand support can change which table entries are used. Numerical cancellation within a permitted component is a separate question, since finite arithmetic can affect that coefficient\textquotesingle s value without changing the table.

Our PHG discipline records the algebra, basis, and support so these relationships remain checkable across generated operations. A support-preserving elaboration excludes out-of-support updates, and its lowering must preserve their absence. Existing restricted parameterizations and kernels also preserve support. We give that fact an explicit identity across the program boundaries that use it.

\subsection*{2.3 Support and Numerical Realization}\label{support-and-numerical-realization}
\addcontentsline{toc}{subsection}{2.3 Support and Numerical Realization}

A tangent slot can carry the same permitted blade set as its parameter. An estimator formed as a linear combination of those tangents stays in that set, provided its scalar operations and update are well-defined. This gives a support invariant for the chosen training construction.

The dimensional type of a gradient is a separate matter. For loss dimension \(d_L\) and parameter dimension \(d_W\), the gradient has dimension \(d_L/d_W\). The learning rate or preconditioner must make the update term have dimension \(d_W\). A coefficient stored in a permitted blade can still violate that dimensional relationship, exceed a numeric range, or violate a rotor constraint. Our joint-constraint mechanism checks each applicable obligation. An update that passes ordinary tensor-shape checks can therefore be rejected before it changes the model, avoiding further training on a state that violates the domain\textquotesingle s constraints.

\subsection*{2.4 Multi-Tangent Forward Gradients and the Gram Matrix}\label{multi-tangent-forward-gradients-and-the-gram-matrix}
\addcontentsline{toc}{subsection}{2.4 Multi-Tangent Forward Gradients and the Gram Matrix}

Support determines which parameter coordinates can change. The tangent strategy determines how much derivative information we obtain about those coordinates per pass. To compare update directions across physical quantities, choose a positive reference scale \(s_j\) with the same units as each parameter \(w_j\), then optimize in dimensionless coordinates \(z_j=w_j/s_j\). The following formulas use a Euclidean metric on these coordinates and a dimensionless loss, obtained by a declared loss scale where needed. This optimizer metric is separate from the Clifford algebra\textquotesingle s metric. Changing the reference scales changes the meaning of a step. Our companion DBC paper, §8.1, develops the corresponding learning-rate units.

For a gradient \(g\in\mathbb R^n\) and a sampled tangent \(v\), the forward-gradient estimator has the form \(\hat g=(g^\top v)v\). For fixed \(g\), its expectation is \(\mathbb E[vv^\top]g\). Thus a zero-mean distribution with covariance \(I\) gives an unbiased estimator, while a uniformly sampled unit direction requires multiplying the estimator by \(n\). In adaptive use, this isotropy must hold conditionally on the history fixing the current gradient. The sampling distribution and scaling are part of the algorithm\textquotesingle s contract.

Flügel et al. \hyperref[ref-flugel-2026]{{[}6{]}} analyze combinations of multiple forward tangents. For linearly independent columns of \(V\), the Euclidean projection onto \(U=\operatorname{span}(V)\) is:

\[P_Ug=V(V^\top V)^{-1}V^\top g.\]

Consider a pressure calibration model with two trainable parameters: an offset and a gain. In the reference-scaled coordinates above, suppose the loss gradient is \(g=(2,1)\). Measuring the offset direction \(v=(1,0)\) gives a directional derivative of \(2\) and the projection \((2,0)\). This is the closest estimate to \(g\) along the measured direction, but it says nothing about the gain derivative. If the full gradient lies in the measured subspace \(U\), the projection recovers it exactly.

Now choose the offset or gain direction with equal probability, keeping the loss and model parameters fixed. The unscaled projections are \((2,0)\) and \((0,1)\), whose average is \((1,1/2)\). Multiplying each estimate by two makes their average equal to \(g\). This is what unbiasedness requires: the estimates recover the true gradient on average over the sampling rule. The doubled estimates are no longer the closest choices along their selected directions. Our training construction must therefore specify the sampling rule and scaling for the property it requires.

The missing-information problem remains separate. Measurements \(V^\top g\) are identical for \(g\) and \(g+w\) whenever \(w\perp U\). In the example, measuring only the offset cannot distinguish a gain derivative of \(1\) from one of \(100\). A bound on full-gradient error therefore needs enough measurements to resolve the unmeasured component, or separate evidence limiting its size.

The matrix \(V\) contains \(k\) independent tangent columns, so \(1\le k\le n\). Its Gram matrix \(V^\top V\) has \(k^2\) entries. Forming it directly costs \(O(nk^2)\) arithmetic for tangent dimension \(n\), and a dense solve adds \(O(k^3)\) work. Storing explicit tangents takes \(O(nk)\) space, although generation and reuse strategies can alter that requirement. A quire can cover its dot products when the accumulation bounds permit. For independent columns, \(\kappa_2(V^\top V)=\kappa_2(V)^2\): forming normal equations squares the spectral condition number. Factoring \(V\) directly into an orthogonal-column factor and an upper-triangular factor (QR) offers an alternative projection computation. The implementation must account for error in the directional measurements \(V^\top g\) as well as its factorization and solve. Near-orthogonality depends on the sampling regime and the actual draw.

A more accurate solve improves the computed projection inside \(U\). Error outside \(U\) still requires additional information, and any curvature or posterior approximation has a separate error contract. A stopping rule must distinguish all four: derivative measurement, projection computation, the unobserved component, and the selected curvature or posterior approximation.

\subsection*{2.4.1 A Metric for the Admitted Program}\label{a-metric-for-the-admitted-program}
\addcontentsline{toc}{subsection}{2.4.1 A Metric for the Admitted Program}

Directional derivatives tell us how a loss changes. Choosing an update also requires a way to compare the sizes of parameter changes. The Euclidean metric above makes that choice through reference scales. By contrast we give our training construction a metric associated with an explicitly selected probabilistic program.

Write its trace presentation as \(F=(\Omega,q,o)\) at dimensionless parameters \(\theta\), with density \(q_\theta\) and output \(o_\theta\). For a parameter direction \(u\), set \(S_u=D_\theta\log q_\theta[u]\) and \(V_u=D_\theta o_\theta[u]\). Using a positive-semidefinite output form \(h\) and \(\kappa>0\), take the score-\/-velocity form of Karwowski et al., §§4 and 8.5 \hyperref[ref-karwowski-2026]{{[}7{]}}:

\[G_\theta(u,v)=\mathbb E_{q_\theta}\left[S_uS_v+\kappa h_{o_\theta}(V_u,V_v)\right].\]

We require a normalized differentiable trace law with locally fixed support, finite displayed second moments and the regularity needed for the chosen derivatives and expectations. The output scales, \(h\) and \(\kappa\) make both terms dimensionally compatible. These are premises of the admitted model. They do not follow from its tensor shapes. The matrix \(G\) can be singular when some parameter directions have no measured effect. We apply the construction to model fragments with established jet semantics; repeating admitted training steps does not by itself supply a geometric semantics for arbitrary loops or recursion.

Our optimizer contract names the presentation and the execution information used by this form. For example, adding an independent draw with a parameter-dependent distribution changes the trace score even when its result is unused. An implementation that removes the draw must either preserve the declared source metric through a justified calculation or admit a changed metric. Retaining source geometry as metadata does not establish that it equals the geometry of the transformed execution. Section 7 carries this distinction into evidence reuse.

\subsection*{2.4.2 The Restricted Natural-Gradient Solve}\label{the-restricted-natural-gradient-solve}
\addcontentsline{toc}{subsection}{2.4.2 The Restricted Natural-Gradient Solve}

Let \(U=[u_1\ \cdots\ u_k]\) contain independent directions in the admitted parameter tangent space. This notation separates parameter directions from the output velocities \(V_u\). For a differentiable loss \(\ell\), define

\[K_{ij}=G_\theta(u_i,u_j),\qquad b_i=D\ell_\theta[u_i],\qquad Ka=b,\qquad \delta=Ua.\]

Both \(K\) and \(b\) can be evaluated through the selected forward directions. Each execution contributes score products and output-velocity pairings to \(K\). The expectation must be evaluated analytically or estimated under a declared sampling procedure. This construction supplies a metric calculation in addition to the directional loss measurements; the Euclidean Gram matrix \(U^{\mathsf T}U\) generally cannot replace \(K\).

\textbf{Proposition 2.1 (restricted metric update).} If \(K\) is positive definite, \(\delta=UK^{-1}b\) is the unique vector in \(\operatorname{span}(U)\) minimizing

\[\frac12G_\theta(\delta,\delta)-D\ell_\theta[\delta].\]

Its negative is a first-order descent direction whenever \(b\ne0\). If a full natural gradient \(g_G\) satisfies \(G_\theta(g_G,v)=D\ell_\theta[v]\) for every parameter tangent \(v\), then \(\delta\) is its \(G_\theta\)-orthogonal projection onto \(\operatorname{span}(U)\).

\emph{Proof.} Substituting \(\delta=Ua\) gives \(\tfrac12a^{\mathsf T}Ka-b^{\mathsf T}a\). Its derivative is \(Ka-b\), and positive definiteness makes its minimizer unique. Moreover, \(D\ell_\theta[-\delta]=-b^{\mathsf T}K^{-1}b<0\) for \(b\ne0\). For every \(Uc\) in the subspace, \(G_\theta(g_G-\delta,Uc)=b^{\mathsf T}c-a^{\mathsf T}Kc=0\). \(\square\)

An invertible change of tangent basis \(U\mapsto UR\) gives \(K\mapsto R^{\mathsf T}KR\) and \(b\mapsto R^{\mathsf T}b\), leaving \(\delta\) unchanged. Changing the subspace changes the question being solved. The restricted vector is not generally an unbiased estimate of the full natural gradient, and inverting an estimated \(K\) does not preserve unbiasedness of its entries.

For singular \(K\), a minimizer exists exactly when \(b\) annihilates \(\ker K\). Solutions then differ by \(\ker K\), and the corresponding parameter updates differ by \(G\)-null vectors in \(U\ker K\). If the condition fails, the displayed objective is unbounded below along a null direction. Adding a positive damping term defines a different, regularized objective with its own reference metric. A finite update also needs a step-size rule; nonlinear admissibility constraints require a constraint-preserving update or retraction, beyond the tangent calculation.

\subsection*{2.4.3 A Two-Parameter Calibration}\label{a-two-parameter-calibration}
\addcontentsline{toc}{subsection}{2.4.3 A Two-Parameter Calibration}

Consider a pressure calibration after scaling the covariate, response and parameters to dimensionless coordinates. Let \(\theta=(\alpha,\beta)\) describe offset and gain. At fixed covariates \(0\) and \(1\), use independent unit-variance Gaussian observations with means \(\alpha\) and \(\alpha+\beta\). Choose the density presentation

\[q_\theta(z_0,z_1)=\varphi(z_0-\alpha)\varphi(z_1-\alpha-\beta),\qquad o_\theta(z_0,z_1)=(z_0,z_1),\]

where \(\varphi\) is the standard-normal density. Holding the trace fixed gives \(V=0\). For residuals \(r_0=z_0-\alpha\) and \(r_1=z_1-\alpha-\beta\), the score components are \((r_0+r_1,r_1)\). Independence and unit variance therefore give

\[G=\begin{pmatrix}2&1\\1&1\end{pmatrix}.\]

With observed responses \((-1,-1)\), the negative log-likelihood up to a constant is

\[\ell(\alpha,\beta)=\tfrac12\left[(\alpha+1)^2+(\alpha+\beta+1)^2\right].\]

At \((0,0)\), take \(U=I\). The forward loss derivatives give \(b=(2,1)^{\mathsf T}\), and the restricted solve gives \(a=(1,0)^{\mathsf T}\). The unit descent step reaches \((-1,0)\), fitting both responses. This exact step is a property of this quadratic loss, whose Hessian equals the selected Fisher matrix. It does not identify Fisher information with a general loss Hessian.

If only the gain direction is measured, \(U=(0,1)^{\mathsf T}\) gives \(K=1\), \(b=1\) and a unit descent step to \((0,-1)\). The second response is fitted while the first residual remains. Thus even an exact metric and solve cannot supply an unmeasured offset update. This is the geometric counterpart of the missing-direction example in §2.4.

\subsection*{2.4.4 Error in the Computed Direction}\label{error-in-the-computed-direction}
\addcontentsline{toc}{subsection}{2.4.4 Error in the Computed Direction}

The optimizer must account for errors in both its metric and its loss measurements. Fix the model, parameter point, tangent basis and reference coordinates. Suppose the exact restricted matrix satisfies \(\lambda_{\min}(K)\ge m>0\). Let \(\widehat K,\widehat b\) obey

\[\|\widehat K-K\|_2\le\epsilon_K,\qquad \|\widehat b-b\|_2\le\epsilon_b,\]

and let \(\rho\ge\|\widehat b-\widehat K\widehat a\|_2\) be a validated solve-residual bound. Since \(Ka=b\),

\[K(\widehat a-a)=(\widehat K-K)(-\widehat a)+(\widehat b-b)-(\widehat b-\widehat K\widehat a),\]

so

\[\|\widehat a-a\|_2\le\frac{\epsilon_b+\rho+\epsilon_K\|\widehat a\|_2}{m},\qquad
\|U\widehat a-\delta\|_2\le\|U\|_2\frac{\epsilon_b+\rho+\epsilon_K\|\widehat a\|_2}{m}.\]

The enclosure includes the arithmetic used to measure its residual. A symmetric \(\widehat K\) is positive definite when \(\epsilon_K<m\), but a small residual alone establishes neither that premise nor a bound on the unknown metric error. Sampling bounds for \(\epsilon_K\) and \(\epsilon_b\) require an explicit probability model; when separately established with failure budgets, their sum bounds failure of the joint event. Adaptive reuse requires the corresponding conditional or time-uniform guarantees. These bounds address the restricted direction, leaving the unmeasured component and finite-step behavior as separate obligations.

The premises also support explicit admission checks. If a symmetric \(\widehat K\) has a certified lower bound \(\underline\lambda\le\lambda_{\min}(\widehat K)\), then

\[\lambda_{\min}(K)\ge\underline\lambda-\epsilon_K.\]

A positive lower enclosure of this difference supplies \(m\) for the preceding error bound. This check still requires justified metric-error evidence. The sign of the loss change has its own check, which can succeed independently of metric accuracy.

\textbf{Corollary 2.2 (descent of a computed direction).} Under the derivative-error bound above, let \(d_{\rm lo}\le\widehat b^{\mathsf T}\widehat a\) and \(A_{\rm hi}\ge\|\widehat a\|_2\) be certified bounds. If a lower enclosure \(c_{\rm lo}\) satisfies

\[0<c_{\rm lo}\le d_{\rm lo}-\epsilon_b A_{\rm hi},\]

then \(D\ell_\theta[-U\widehat a]\le-c_{\rm lo}<0\).

\emph{Proof.} The directional derivative is \(-b^{\mathsf T}\widehat a\). Cauchy-\/-Schwarz gives

\[b^{\mathsf T}\widehat a\ge\widehat b^{\mathsf T}\widehat a-\|b-\widehat b\|_2\|\widehat a\|_2\ge d_{\rm lo}-\epsilon_b A_{\rm hi}\ge c_{\rm lo}.\]

Negating this inequality proves the claim. \(\square\)

For an exact scalar example, take \(K=\widehat K=1\), \(\widehat b=\widehat a=1/100\), \(b=-1/100\) and \(\epsilon_b=1/50\). The solve residual is zero and the direction-error enclosure holds, but \(D\ell[-\widehat a]=1/10000>0\). The descent test refuses acceptance because its lower bound is \(-1/10000\). With \(b=\widehat b\) and \(\epsilon_b=0\), the same computed direction has certified descent margin \(1/10000\).

The runtime calculation encloses the dot product, norm and final comparison, with the derivative-error evidence bound to the current model, parameter point and tangent basis. Failure to certify descent can trigger additional measurements or a different admitted direction; it does not prove ascent. Statistical certificates retain their joint-event and adaptive-use conditions. The result certifies first-order descent of the represented direction. Its native realization and a finite step retain their arithmetic, step-size and constraint-preservation obligations.

Our resource account now includes \(k\) score measurements, their output velocities and a \(k\times k\) metric calculation, along with loss derivatives and solver workspace. Streaming executions can accumulate the restricted matrix without constructing the full parameter metric, but does not remove the cost of velocity pairings or their live storage. The chosen derivative realization determines whether traces are retained, reconstructed or recomputed. The workspace and recovery contracts in our companion \emph{Negative and Fractional Types} paper \hyperref[ref-haynes-nft]{{[}13{]}} govern exact restoration when that route is selected; a geometric error bound alone cannot justify resource elimination.

Tensor shape and the permitted parameter support determine \(n\). Physical dimensions such as length determine different type relationships. When \(n\) and \(k\) are known, our proposed editor feedback can report the ratio and the associated resource calculations. Expected approximation quality requires the sampling assumptions as well. Restricting the coordinates reduces the dimension used in these storage and arithmetic counts. Section 9 specifies comparisons of the resulting training configurations at matched prediction quality.

\subsection*{2.5 Elimination Growth and Numerical Evidence}\label{elimination-growth-and-numerical-evidence}
\addcontentsline{toc}{subsection}{2.5 Elimination Growth and Numerical Evidence}

The choices above determine how much derivative information we collect and how much work we devote to it. To use that information reliably, we also need to know how errors arise inside the calculations. The projection in §2.4 and the posterior computation developed in §3 both involve solving linear equations. These give us a concrete place to connect a required accuracy to the arithmetic that can deliver it.

In an \(LU\) factorization, the coefficient matrix is written as \(A=LU\), with a lower-triangular factor \(L\) and an upper-triangular factor \(U\). These triangular systems can be solved one unknown at a time. Entries in the factors can be much larger than those in \(A\), increasing the required numerical range and affecting rounding error. With \(L\) normalized to have unit diagonal, Urschel \hyperref[ref-urschel-2026]{{[}9{]}} defines

\[
\operatorname{growth}(A)=\max\left\{\|L\|_{\max},
\frac{\|U\|_{\max}}{\|A\|_{\max}}\right\},
\]

where \(\|\cdot\|_{\max}\) is the largest absolute entry. The required pivots must be nonzero. This definition bounds the factors. A range analysis covering every elimination intermediate additionally needs an applicable intermediate-entry bound.

For an \(n\times n\) matrix \(G\) of independent standard Gaussian entries, let \(\Pi(G)\) denote the row-permuted matrix obtained by partial pivoting. Urschel\textquotesingle s Theorem 1.4 \hyperref[ref-urschel-2026]{{[}9{]}} gives, for each \(\alpha>0\), constants \(C_\alpha,n_\alpha\) such that

\[
\Pr\!\left[\operatorname{growth}(\Pi(G))>
\sqrt n\exp(C_\alpha\sqrt{\log n})\right]<n^{-\alpha},
\qquad n\ge n_\alpha.
\]

The upper growth scale is \(n^{1/2+o(1)}\). The constants and dimension threshold in this bound must be made usable before it supplies a risk budget for a particular small matrix. A Gaussian prior or Gaussian observation noise does not make a posterior precision matrix, or the Gram matrix in §2.4, an independent-entry Gaussian matrix.

Urschel\textquotesingle s Theorem 1.2 \hyperref[ref-urschel-2026]{{[}9{]}} provides a different route. Fix any nonsingular real \(n\times n\) matrix \(A\), and draw \(Q\) from Haar measure on the orthogonal group, the probability law invariant under orthogonal multiplication. For \(0<\delta<1\), elimination without pivoting on \(QA\) satisfies

\[
\Pr\!\left[\operatorname{growth}(QA)>
\frac{14n^{3/2}\sqrt{\log(en)}}{\delta}\right]<\delta.
\]

The systems \(Ax=b\) and \((QA)x=Qb\) have the same exact solution, and \(\kappa_2(QA)=\kappa_2(A)\), where \(\kappa_2(A)=\|A\|_2\|A^{-1}\|_2\) is the spectral condition number. Thus the transformation controls elimination growth while preserving the problem\textquotesingle s spectral sensitivity. The denominator in this growth factor is \(\|QA\|_{\max}\). An absolute magnitude budget needs that transformed-input bound as well.

A third route, given by Theorem 1.6 in Urschel \hyperref[ref-urschel-2026]{{[}9{]}}, treats partial pivoting on \(A+\varepsilon G\) for \(\|A\|_2\le1\) and \(0<\varepsilon<1\). It gives a smoothed bound with a different dependence on dimension and perturbation scale. The perturbation needs a statistical or algorithmic justification: adding independent noise to a precision matrix changes the inference problem and can destroy symmetry or positive definiteness. We retain this route for constructions that can account for those changes.

Urschel\textquotesingle s results give the numerical part of our proposal a firmer foundation: a declared matrix construction and elimination procedure can carry a proved bound on the probability of excessive factor growth. Our original ADM design already treated numerical precision and resource requirements as properties to establish within the model\textquotesingle s construction. That choice gives these new bounds a practical role in selecting arithmetic and retaining numerical evidence as a model learns and enters service.

To make that evidence usable, we must connect these exact-arithmetic results to the factors, pivots and conversions of the solve that actually runs. Rounding may change pivot choices. A finite-machine Haar sampler and the products \(QA,Qb\) also need error bounds. Posit and fixed-point realizations require their actual precision and range contracts. A uniform floating-point error model cannot be assumed for them.

Growth control can contribute to a backward-error bound, which describes a nearby system solved by the computed answer. Sensitivity then relates that nearby problem to the intended one. A direct residual gives another useful check: if \(x\) solves \(Ax=b\), \(\widehat x\) is the computed result, and \(r=b-A\widehat x\), then

\[
x-\widehat x=A^{-1}r,\qquad
\|x-\widehat x\|_2\le\|A^{-1}\|_2\|r\|_2.
\]

The residual must itself be enclosed, including evaluation error and any difference between intended and stored inputs. Sensitivity supplies the remaining factor in the forward bound. Coverage of machine intermediates needs its own range argument. For Haar elimination without pivoting, the growth definition bounds both \(L\) and \(U\). An error theorem that assumes partial pivoting\textquotesingle s bounded multipliers cannot be transferred unchanged.

A required output tolerance gives these results an operational use. Growth and intermediate-range bounds inform carrier capacity, while the realized solve and its sensitivity determine whether the answer meets the error budget. A candidate arithmetic configuration can then be evaluated for the cost of producing that answer. Support analysis removes excluded operations, and numerical evidence establishes the accuracy of the retained computation. Section 3.2 uses that accuracy to decide when further numerical work is unnecessary.

\section*{3. A Proposed Bayesian Training Model}\label{a-proposed-bayesian-training-model}
\addcontentsline{toc}{section}{3. A Proposed Bayesian Training Model}

\subsection*{3.1 Structural Constraints and Posterior Updates}\label{structural-constraints-and-posterior-updates}
\addcontentsline{toc}{subsection}{3.1 Structural Constraints and Posterior Updates}

A model for the physical property of pressure can have a valid physical interface while remaining uncertain about its coefficients. For example, a linear map from velocity to pressure gradient gives its coefficients dimension \(\mathrm{Pa\,s\,m^{-2}}\). That identity fixes their dimensional role. A prior supplies the distribution over their admissible magnitudes, and observations update it through a likelihood.

Making a quantity\textquotesingle s physical meaning, admissible values and numerical realization jointly govern its computation was a principal motivation for our Dimensional Type System. ADM extends that design to retain the probability model alongside dimensional and geometric constraints. Our architectural contribution is to let learned uncertainty inform representation and resource choices while preserving the conditions under which the computation remains valid.

Within this account, principal dimensional inference establishes the most general valid measure relationship. A range or support declaration restricts admitted values. Probability assigns weight within that space, so models with identical structural constraints can express different beliefs.

The prior describes uncertainty before the observations. The likelihood describes those observations given a parameter. For admissible parameter space \(\Theta_{\mathrm{adm}}\), prior density \(p(\theta)\) relative to a specified measure \(\mu\), and likelihood \(p(\mathcal D\mid\theta)\), Bayes\textquotesingle{} rule gives:

\[
p(\theta\mid\mathcal D)=
\frac{p(\mathcal D\mid\theta)p(\theta)}
{\int_{\Theta_{\mathrm{adm}}}p(\mathcal D\mid u)p(u)\,d\mu(u)}.
\]

The denominator must be finite and positive. Updating changes the relative weight of admissible parameters while keeping the same declared space. The posterior carries the uncertainty needed for prediction and subsequent decisions, including cases where a new observation increases uncertainty.

The posterior construction determines which numerical tools an update needs. Forward directional derivatives supply first-order information for optimizing an approximate posterior. Section 2.4 supplies an additional, declared program metric and a restricted natural-gradient solve. This does not infer a Hessian or posterior covariance from first derivatives. Our worked example below uses a linear-Gaussian model whose posterior can be computed through factorizations and solves. This lets us follow a complete update from its prior and observations to a prediction, then establish the numerical accuracy required to use that prediction.

\subsection*{3.1.1 A Linear-Gaussian Update}\label{a-linear-gaussian-update}
\addcontentsline{toc}{subsection}{3.1.1 A Linear-Gaussian Update}

Consider two unknown coefficients in declared, reference-scaled coordinates. Their physical units and any geometric support have already been checked. Let the prior and observation model be

\[
\theta\sim N(m_0,C_0),\qquad y=H\theta+\epsilon,
\qquad \epsilon\sim N(0,R),
\]

with independent \(\theta\) and \(\epsilon\), known \(H\), and symmetric positive-definite \(C_0,R\). The densities use Lebesgue measure on these coordinates. Completing the square gives an exact Gaussian posterior with precision \(\Lambda\), covariance \(C\), and mean \(m\):

\[
\begin{aligned}
\Lambda&=C_0^{-1}+H^\top R^{-1}H,&
 h&=C_0^{-1}m_0+H^\top R^{-1}y,\\
\Lambda m&=h,& C&=\Lambda^{-1}.
\end{aligned}
\]

The inverses specify the mathematics. An implementation uses appropriate solves or factorizations. This precision matrix has dependent entries even if \(H\) was sampled with independent Gaussian entries, so Urschel\textquotesingle s Gaussian partial-pivoting bound (Theorem 1.4) \hyperref[ref-urschel-2026]{{[}9{]}} does not apply. Its positive-definite structure favors a native Cholesky or square-root formulation. Fresh Haar preconditioning gives an alternative general solve covered by Theorem 1.2 in the same paper \hyperref[ref-urschel-2026]{{[}9{]}}. A comparison must account for the transformation cost and loss of symmetry.

For a small exact instance, choose \(m_0=0\), \(C_0=R=I_2\) and

\[
H=\begin{pmatrix}1&0\\1&1\end{pmatrix},\qquad
y=\begin{pmatrix}1\\2\end{pmatrix}.
\]

\[
\begin{aligned}
\Lambda&=\begin{pmatrix}3&1\\1&2\end{pmatrix},&
 h&=\begin{pmatrix}3\\2\end{pmatrix},\\
m&=\begin{pmatrix}4/5\\3/5\end{pmatrix},&
 C&=\frac15\begin{pmatrix}2&-1\\-1&3\end{pmatrix}.
\end{aligned}
\]

Since \(\Lambda=I_2+H^\top H\), its smallest eigenvalue is at least one. An enclosed mean residual \(r=h-\Lambda\widehat m\) therefore gives \(\|m-\widehat m\|_2\le\|r\|_2\). Covariance needs separate evidence: for \(E=I_2-\Lambda\widehat C\), the identity \(C-\widehat C=\Lambda^{-1}E\) bounds its error. A stored covariance must also satisfy the required symmetry and positive-definiteness checks, for example through a validated factorization. The mean residual alone supplies neither property.

For prediction, take a new scalar observation \(X=\theta_1+\xi\), with independent \(\xi\sim N(0,3/5)\). Its posterior predictive law is \(N(4/5,1)\): the variance includes both coefficient uncertainty \(C_{11}=2/5\) and observation variance \(3/5\). We freeze this law for the monitoring window below. Its parameters are exact illustrative values. A numerical implementation must establish the mean and variance enclosures used by the decision.

The prediction also determines how much posterior computation is needed. A Cholesky factor \(\Lambda=LL^\top\) can serve both the mean solve and predictive-variance queries. For a prediction direction \(a\), solving \(Lu=a\) gives \(a^\top Ca=\|u\|_2^2\). Here \(a=e_1\), so the implementation can obtain \(C_{11}\) without constructing or storing the full covariance matrix. It reuses one factorization and computes the requested uncertainty. Numerical evidence then covers the factor, the solves and the resulting scalar variance. An interface that requests the full covariance retains the matrix obligations above.

\subsection*{3.2 Distribution Shift and an Error-Aware Trigger}\label{distribution-shift-and-an-error-aware-trigger}
\addcontentsline{toc}{subsection}{3.2 Distribution Shift and an Error-Aware Trigger}

We now use the posterior prediction to decide whether new observations justify preparing a replacement. A discrepancy can reflect a difference between the deployment process and the model, sampling variation within the observation window, or numerical error in the calculation. Our decision accounts for sampling and numerical uncertainty before using that discrepancy to justify preparing a candidate. We use Kullback-Leibler (KL) divergence, a dimensionless measure of discrepancy between probability laws on the same outcomes, and carry the posterior\textquotesingle s numerical evidence into its evaluation.

Freeze the predictive reference \(p_0=N(m_1,1)\) from §3.1.1, where \(m_1=4/5\). For this monitoring window, assume \(X_1,\ldots,X_N\) are independent with deployment law \(q_\mu=N(\mu,1)\), for a fixed unknown \(\mu\). These observations are independent of the data used to establish the reference. The equal known variance is an explicit monitoring assumption. Repeated observations sharing a latent coefficient do not automatically have this independence. Other deployment processes require their own estimator and coverage argument.

The population discrepancy and an exact finite-window estimator are

\[
D=D_{\mathrm{KL}}(q_\mu\|p_0)=\frac{(\mu-m_1)^2}{2},\qquad
D_{\mathrm{est}}=\frac{(\overline X-m_1)^2}{2}.
\]

Let \(\widehat z\) approximate \(z=\overline X-m_1\) with an established bound \(|\widehat z-z|\le e\). Let \(\widehat D\) approximate \(\widehat z^2/2\) with evaluation error at most \(e_{\mathrm{eval}}\). Writing \(d=|\widehat z|\) gives the numerical enclosure

\[
|D_{\mathrm{est}}-\widehat D|\le e_{\mathrm{num}},\qquad
 e_{\mathrm{num}}=de+e^2/2+e_{\mathrm{eval}}.
\]

For \(r=z_{1-\delta_{\mathrm{stat}}/2}/\sqrt N\), where \(z_a\) is the standard-normal quantile, the sampling event \(|\mu-\overline X|\le r\) has probability \(1-\delta_{\mathrm{stat}}\). On this event, the difference of squares gives

\[
|D-D_{\mathrm{est}}|\le e_{\mathrm{stat}},\qquad
 e_{\mathrm{stat}}=(d+e)r+r^2/2.
\]

When both bounds hold, adding the two radii gives an interval for the population discrepancy:

\[
D\in\left[\max(0,\widehat D-e_{\mathrm{stat}}-e_{\mathrm{num}}),
                \widehat D+e_{\mathrm{stat}}+e_{\mathrm{num}}\right].
\]

If the numerical bound has failure probability at most \(\delta_{\mathrm{num}}\), the combined failure probability is at most \(\delta_{\mathrm{stat}}+\delta_{\mathrm{num}}\). This union bound requires both guarantees for the specified process, but no independence between their failure events. The numerical budget includes the implemented solver and statistic, beyond any exact-arithmetic growth event used to justify them.

For an illustrative decision policy, take \(N=10{,}000\), \(r=0.03\), and threshold \(\tau=0.1\). The sampling failure probability is \(2[1-\Phi(3)]\approx0.0027\), with \(\Phi\) the standard-normal distribution function. Suppose the accepted numerical tolerances are \(e=10^{-6}\) and \(e_{\mathrm{eval}}=10^{-8}\). These are requirements for a future implementation, not measured results. A mean-residual enclosure of \(5\times10^{-7}\) in §3.1.1 and a sample-mean/subtraction error bound of the same size would supply the specified \(e\).

{\def\LTcaptype{none} 
\begin{longtable}[]{@{}
  >{\raggedright\arraybackslash}p{(\linewidth - 6\tabcolsep) * \real{0.2500}}
  >{\raggedright\arraybackslash}p{(\linewidth - 6\tabcolsep) * \real{0.2500}}
  >{\raggedright\arraybackslash}p{(\linewidth - 6\tabcolsep) * \real{0.2500}}
  >{\raggedright\arraybackslash}p{(\linewidth - 6\tabcolsep) * \real{0.2500}}@{}}
\toprule\noalign{}
\begin{minipage}[b]{\linewidth}\raggedright
Computed displacement \(d\)
\end{minipage} & \begin{minipage}[b]{\linewidth}\raggedright
Nominal \(\widehat D\)
\end{minipage} & \begin{minipage}[b]{\linewidth}\raggedright
Outward-rounded enclosure for \(D\)
\end{minipage} & \begin{minipage}[b]{\linewidth}\raggedright
Decision at \(\tau=0.1\)
\end{minipage} \\
\midrule\noalign{}
\endhead
\bottomrule\noalign{}
\endlastfoot
\(0.20\) & \(0.02000\) & {[}0.01354, 0.02646{]} & No trigger \\
\(0.50\) & \(0.12500\) & {[}0.10954, 0.14046{]} & Prepare a candidate \\
\(0.45\) & \(0.10125\) & {[}0.08729, 0.11521{]} & Undecided \\
\end{longtable}
}

The policy triggers when the lower endpoint exceeds \(\tau\), and returns no trigger when the upper endpoint is at most \(\tau\). An overlapping interval requires more observations, greater numerical accuracy, or the declared fallback. Under this observation model, increasing the window reduces the sampling radius, while more accurate calculation reduces the numerical radius. Treating a computed predictive variance as exactly one would also require its separate covariance evidence.

The enclosure directs the next expenditure of work. Once the interval lies wholly on one side of the threshold, greater arithmetic precision is unnecessary for this decision at the stated coverage. For \(d=0.45\), the sampling contribution is approximately \(0.01395\), while the numerical contribution is below \(5\times10^{-7}\). Even eliminating numerical error leaves the interval across the threshold. Additional observations or the declared fallback address that uncertainty. Separating the two contributions avoids spending more arithmetic on a decision limited by the available observations.

The worked monitor uses Gaussian laws with positive densities on the real line. Extending it to other observation models requires compatible distributions and support conditions. A raw empirical measure on continuous samples cannot simply replace one of those densities.

Having identified what would justify preparing a candidate, we can use the remaining uncertainty to guide further measurement. An unresolved result may call for more observations before another decision, while our monitoring policy determines when to reassess a model that remains in service. Repeated monitoring needs a time-uniform procedure or an allocated risk budget across windows. The single-window probability above is not a run-wide guarantee.

Candidate evaluation then asks whether the proposed replacement improves the deployment\textquotesingle s chosen quality measures. Continued service during an undecided or triggered window follows the existing operational policy. The value of any additional evidence depends on what was observed and how that observation relates to the model\textquotesingle s parameters.

\subsection*{3.3 Consumer-Supplied Evidence and Operational Provenance}\label{consumer-supplied-evidence-and-operational-provenance}
\addcontentsline{toc}{subsection}{3.3 Consumer-Supplied Evidence and Operational Provenance}

The monitoring example uses sensor observations, but an ADM may also receive a user\textquotesingle s judgment about a prediction. These sources of evidence require different observation models. Using a preference judgment in a likelihood requires a model relating that judgment to parameters, including judgment noise. Its meaning differs from a sensor measurement of physical accuracy.

The resulting update may affect parameters beyond the observed input region because those parameters are shared across the model. The provenance record should retain the observation, its context, and the likelihood assumptions used. This permits later review when a sensor calibration or interpretation of feedback changes.

Our versioning design connects evidence provenance to the candidate artifact and its admission results. A record can include the trigger statistic, observation window, and source identifiers, with sensitive data retained under the deployment\textquotesingle s access policy. Authenticating that record establishes its origin and integrity under the chosen trust assumptions. It does not make every recorded observation true.

\subsection*{3.4 Bayesian Distillation from General-Purpose Models}\label{bayesian-distillation-from-general-purpose-models}
\addcontentsline{toc}{subsection}{3.4 Bayesian Distillation from General-Purpose Models}

A general-purpose model may help initialize a domain model by supplying candidate relationships, examples, or hypotheses. We propose \emph{Bayesian distillation} as an investigation of how that information can be converted into an explicit prior over admissible domain parameters. Structural checks would reject incompatible units and support. The statistical quality of the retained information still requires evaluation.

Bayesian teaching by Qiu et al. \hyperref[ref-qiu-2026]{{[}14{]}} motivates investigating structured transfer from another model. Section 8.1 distinguishes their evaluated belief-updating task from our proposed prior construction.

A distillation procedure must specify what is observed from the source model and how those observations define a distribution. Outputs from a model\textquotesingle s application programming interface (API), its stated uncertainty, and its internal activations provide different evidence and require different access. Repeated generated answers can be correlated. Treating them as independent domain observations would add a statistical assumption that needs justification.

\textbf{Typed parameterization and prior support.} For a linear layer \(f(x)=Wx\), dimensional inference relates each coefficient to the input and output units. A support declaration restricts the blade coordinates available to a geometric parameter. These coordinates and their declared units define the admissible space. Distillation must supply the prior measure required by §3.1, including a reference measure and finite, positive normalization when it uses a density. The structural rules do not determine the relative weights.

When admissibility defines a lower-dimensional linear subspace, multiplying an ambient continuous density by an indicator can leave total probability zero. A density on the subspace, a specified conditional construction, and projection of samples are distinct modeling choices. Projection defines a pushforward distribution and generally differs from conditioning. Nonlinear requirements such as unit-rotor structure need their own parameterization as well.

A candidate prior would retain its source evidence and chosen construction. Subsequent observations would update it through the specified likelihood. The experiments we propose in §9.1 would assess whether that initialization improves calibration or data efficiency.

\section*{4. Our Warm Rotation Architecture}\label{our-warm-rotation-architecture}
\addcontentsline{toc}{section}{4. Our Warm Rotation Architecture}

A candidate that passes its checks must enter service without mixing its parameters with the previous version\textquotesingle s state or numerical evidence. We use \emph{warm rotation} for a candidate-model transition that preserves a consistent inference service while the incoming configuration is validated and installed. A request must observe one coherent version. Its model weights, representation, and associated state must agree, including at any cross-device boundary.

We propose using our actor and version infrastructure to separate candidate preparation from the active pathway. The feasibility of concurrent work depends on actual memory and scheduling budgets. Training can also occur on another permitted node, with the resulting artifact checked for the destination before installation.

\subsection*{4.1 Formal Definition}\label{formal-definition}
\addcontentsline{toc}{subsection}{4.1 Formal Definition}

Let a model version be \(M_t=(\theta_t,\mathcal C_t,\mathcal V_t)\), where \(\theta_t\) is its configuration, \(\mathcal C_t\) contains the evidence required for admission, and \(\mathcal V_t\) records provenance and identity. Each item in \(\mathcal C_t\) states its premises and the artifact or operation it covers. Structural proofs, numerical enclosures, and statistical evaluation records retain their distinct meanings.

\begin{quote}
\textbf{Definition 4.1 (Warm Rotation).} A warm rotation \(M_t\to M_{t'}\) requires:

\begin{enumerate}
\def\labelenumi{\arabic{enumi}.}
\tightlist
\item
  Candidate preparation completes within the applicable resource and scheduling policy.
\item
  The candidate meets its structural, numerical, and memory obligations for the destination, and passes the declared statistical acceptance procedure.
\item
  Its version record binds the candidate artifact to its evidence and the deployment\textquotesingle s authentication policy.
\item
  Each inference request observes a coherent version throughout the transition.
\end{enumerate}
\end{quote}

The destination carrier must cover the admitted values, and the selected operations must meet the numerical error requirement. A known uncovered range excludes a candidate even if its typical inputs are smaller. Precision is another condition: a carrier may cover the magnitude while producing an error enclosure too wide for the decision in §3.2. An unresolved requirement can remain pending during preparation, but must be established before the corresponding result is used.

For the worked example, candidate evidence would identify the prior and observations defining \(\Lambda,h\). It would bind the mean and covariance enclosures to the solver and arithmetic used, with the monitoring assumptions and risk allocation attached to the trigger. A changed observation window or numeric realization may require new evidence even when dimensions and support remain the same. Statistical acceptance additionally evaluates the candidate on the deployment\textquotesingle s chosen prediction and calibration criteria.

For training by §2.4, the version also binds the trace presentation, observed execution information, output scales and metric, tangent directions, loss and parameter point. Metric estimates and solve enclosures apply to that combination. An equivalent output law can arise from a different execution presentation, so distributional equivalence alone does not authorize reuse of the optimizer matrix.

The transition protocol must establish item 4. Draining active requests and buffering new ones is one possible design, with queue capacity and latency costs to account for. Versioned request ownership is another possible mechanism. An actor mailbox alone does not prove atomic installation across devices or release of every reference to the previous version. Stateful models also require a migration or retention policy for their operational state.

We extend that version requirement to dependent calls across a composed interaction. A prediction supplied to a language model and a later check that relies on it carry the same ADM version identifier. If rotation occurs between those calls, the service either retains that version for the dependent check or recomputes the dependent results using the new version. The retention policy and its storage cost belong in the request and resource accounting. This preserves the connection between a response, its calculation and the evidence accepted for that calculation.

Signing policy is deployment-specific. A local development record may use content identity and history, while a deployment requiring authenticated provenance adds the relevant signature and key policy. The definition therefore does not require one cryptographic scheme for every use.

\subsection*{4.2 Memory Feasibility Across Hardware Targets}\label{memory-feasibility-across-hardware-targets}
\addcontentsline{toc}{subsection}{4.2 Memory Feasibility Across Hardware Targets}

The memory budget must include the active model as well as the candidate\textquotesingle s training state. We account for disjoint storage at time \(t\) as:

\[
\begin{aligned}
M(t)={}&M_{\mathrm{active}}(t)+M_{\mathrm{candidate}}(t)\\
 &+M_{\mathrm{live\ primal/tangent}}(t)+M_{\mathrm{optimizer}}(t)\\
 &+M_{\mathrm{posterior/curvature}}(t)+M_{\mathrm{solve/check}}(t)
   +M_{\mathrm{buffers}}(t).
\end{aligned}
\]

The peak is \(\max_t M(t)\). The accounting counts physical storage once, allowing sharing or reuse when lifetime analysis justifies it. Forward-mode removes the conventional backward tape term. The remaining total depends on layer width and retained connections, together with the tangent strategy and optimizer. Solver factors, random transforms, and validation workspaces also contribute to the peak. The target must accommodate these allocations and their data movement alongside active inference.

For \(P\) parameters, explicitly storing \(r\) curvature directions takes \(O(Pr)\) coefficients. Relative to one \(P\)-coefficient weight vector, that is an \(O(r)\) term. A structured or implicit representation can change the storage requirement, but needs a concrete algorithm and budget. Directional derivatives alone provide no such representation.

Arithmetic capability is another independent condition. An inference accelerator may lack an operation needed by the candidate update. The platform must offer its native or permitted emulated realization. A quire declaration supplies a finite capacity, whose proof must cover every reached partial sum. Final cancellation does not justify an earlier overflowing partial sum.

Bounded-posit parameters can be evaluated against the model\textquotesingle s sound ranges and precision objective. Observed concentration near a magnitude can motivate a candidate bias or rescaling, followed by coverage and conversion checks. The useful width depends on the model and target contract. Tangent count additionally depends on estimator variance and problem geometry.

Deployment changes can also affect the probability of exceeding a declared operating envelope. Our range account retains that probability and the required boundary behavior alongside deterministic coverage requirements. To relate a tail probability under the modeled distribution to one under deployment, we need evidence about how those distributions differ. Total variation measures their largest probability difference for the same measurable event. If a justified bound gives \(\operatorname{TV}(P_{\mathrm{deploy}},P_{\mathrm{model}})\le\eta\), then for a declared tail event \(E\),

\[
P_{\mathrm{deploy}}(E)\le P_{\mathrm{model}}(E)+\eta.
\]

This can inform a deployment risk policy when both terms are established. The bound still permits the event, so a finite carrier needs defined behavior at its boundary. A small observed KL statistic alone does not establish the required population shift bound.

Support sparsity removes absent arithmetic paths. Precision selection changes the represented values and the cost of the retained operations. Evaluating them together requires prediction quality and total resource use, including checking and fallback work.

\subsection*{4.3 Cluster-Scale Warm Rotation}\label{cluster-scale-warm-rotation}
\addcontentsline{toc}{subsection}{4.3 Cluster-Scale Warm Rotation}

In our warm rotation design, candidate preparation can use capacity across multiple nodes while the active model continues serving requests. We place that work within each node\textquotesingle s memory and scheduling budget, accounting for transfers and preserving the active service\textquotesingle s declared latency limits. The completed candidate is validated for its destination before installation.

For distributed accumulation, exactness requires a compatible numeric representation throughout the reduction. Rounding a worker\textquotesingle s result before transfer has already changed the operand delivered to the next stage. An exact reduction of those delivered values preserves their sum within capacity, with the final conversion governed separately. The version record identifies worker contributions and the reduction procedure so that the corresponding evidence remains attributable to the candidate.

\section*{5. Clifford Neural Networks Under Our Typed Substrate}\label{clifford-neural-networks-under-our-typed-substrate}
\addcontentsline{toc}{section}{5. Clifford Neural Networks Under Our Typed Substrate}

\subsection*{5.1 Support Preservation Through Training}\label{support-preservation-through-training}
\addcontentsline{toc}{subsection}{5.1 Support Preservation Through Training}

The nine-product bivector construction in §2.2 remains valid during training when updates preserve the support that justified it. We obtain that property by constructing update directions from the parameter\textquotesingle s permitted blade coordinates. Linear combinations of those directions remain in the same coordinate space. The following proposition establishes this structural invariant. Estimator accuracy and numerical precision have their separate obligations.

\begin{quote}
\textbf{Proposition 5.1 (Support Invariance of Restricted Updates).} Let \(S\) be a fixed blade-support set in a declared Clifford algebra with a fixed basis. Suppose \(W\) is supported in \(S\), each tangent \(v_i\) is supported in \(S\), and the estimator \(\hat g\) is a defined scalar linear combination of those tangents. Then a defined, dimensionally compatible update \(W'=W-\eta\hat g\) is supported in \(S\). If every blade in \(S\) has the same parity, the update remains within that class.
\end{quote}

\emph{Proof.} The coordinate space supported in \(S\) is closed under scalar linear combinations. Therefore \(\hat g\) is supported in \(S\), as is the difference of \(W\) and the compatible update term. Coefficients may become zero, so the statement preserves an upper bound on support rather than requiring every permitted component to stay nonzero. \(\square\)

The proposition concerns the estimate in the restricted parameterization. An unrestricted ambient gradient can have other components. The support proof applies independently of numeric format, while update dimensions and finite realization retain the obligations of §2.3. A nonlinear rotor constraint needs the additional construction below.

\begin{quote}
\textbf{Corollary 5.2 (Structural Sparsity Stability).} If the declared algebra, basis, and permitted operand support sets remain fixed under valid updates, the basis-product paths excluded by those facts remain excluded throughout training.
\end{quote}

The signature and chosen basis determine the Cayley table. In an orthogonal basis, a product has the blade index given by symmetric difference, with the sign and metric factors determined by the basis rules. A repeated degenerate generator can make a product zero. The support sets determine which operand pairs are possible. Those facts permit a compiler to omit excluded paths, while numerical cancellation may yield further zeros in a particular execution.

Our declared model structure establishes these exclusions before training begins. Corollary 5.2 shows that they remain valid as training updates the coefficients within the permitted support. The nine-product construction in §2.2 can therefore be retained throughout training and when a trained version enters service. Other support sets have their own product counts, which our compiler can report for comparison with the realized operation costs.

\subsection*{5.2 Equivariance and Rotor Constraints}\label{equivariance-and-rotor-constraints}
\addcontentsline{toc}{subsection}{5.2 Equivariance and Rotor Constraints}

For an ADM whose domain requires rotational symmetry, changing the input coordinate frame should produce the corresponding change in its prediction. A layer \(f\) is equivariant under a declared group action when it satisfies:

\[f(\rho_{\mathrm{in}}(g)x)=\rho_{\mathrm{out}}(g)f(x)\]

for every admitted transformation \(g\). Here \(\rho_{\mathrm{in}}\) and \(\rho_{\mathrm{out}}\) describe how that transformation acts on the layer\textquotesingle s inputs and outputs. A unit rotor represents a geometric transformation through a sandwich action. A fixed rotor sandwich does not satisfy this condition for every other rotor merely by having unit norm, since the transformations need not commute.

To establish equivariance, our proof account must identify the two group actions and show that the layer intertwines them as above. Typed grade and support rules provide structural premises. When training updates a rotor parameter, the update must also preserve the intended rotor family. Support restriction supplies the permitted components. Normalization and membership establish the additional geometric conditions.

The normalization equation \(R\widetilde R=1\), with \(\widetilde R\) denoting reversal of the basis-vector products, involves nonlinear real arithmetic. It is distinct from the linear integer obligations used for dimensions and many layout checks. A normalization construction needs a suitable proof or checked library rule, with its domain assumptions retained. The realized finite arithmetic then needs a separate error analysis. A symmetry-error bound must cover the realized operations and conversions.

\subsection*{5.3 Mixed-Algebra and Physics-Structured Architectures}\label{mixed-algebra-and-physics-structured-architectures}
\addcontentsline{toc}{subsection}{5.3 Mixed-Algebra and Physics-Structured Architectures}

A model may combine layers whose geometric representations serve different operations. A conversion between them must preserve the meaning of the transferred quantity. For example, a network can connect projective geometric algebra (PGA) to conformal geometric algebra (CGA) through an explicitly declared conversion. Our PHG would record its source and destination contracts, with a law for each geometric quantity transferred. The contract must identify the point or plane encoding intended by the application, and the numerical conversion must preserve that meaning within its error bound.

For physics-structured layers, dimensions and geometric types check complementary aspects of an operator. The gradient of scalar pressure has pressure-per-length dimension and vector structure. Divergence consumes a suitable vector field and returns a scalar. A substitution between them can therefore expose a type mismatch. A correctly typed operator still needs the mathematical law and implementation required by the physical model.

\section*{6. Spiking Neural Networks and the Temporal Hyperedge}\label{spiking-neural-networks-and-the-temporal-hyperedge}
\addcontentsline{toc}{section}{6. Spiking Neural Networks and the Temporal Hyperedge}

\subsection*{6.1 Temporal Structure in the PHG}\label{temporal-structure-in-the-phg}
\addcontentsline{toc}{subsection}{6.1 Temporal Structure in the PHG}

An ADM can also use the temporal structure of a spiking model to specialize computation. In the impulse model considered here, known decay determines state evolution between events, allowing computation to advance between the events required by the model. Our PHG can retain these temporal rules together with the state needed to apply them. The resulting work depends on arrivals and current state. The geometric exclusions of §5 remain fixed under support-preserving updates.

Our dimensional discipline can check that a membrane time constant has units of time and that a firing rate has inverse-time units. A mismatch between expected and observed frequencies requires numeric or statistical evidence in addition to those units. Temporal contracts can carry event-window and state requirements through the same graph, with target-specific energy and timing evaluated separately.

\subsection*{6.2 Coincidence Detection as a Hyperedge}\label{coincidence-detection-as-a-hyperedge}
\addcontentsline{toc}{subsection}{6.2 Coincidence Detection as a Hyperedge}

Consider the impulse-driven Leaky Integrate-and-Fire equation, with voltage measured relative to the resting potential:

\[\tau_m\frac{dV_m}{dt}=-V_m+\sum_i q_i\sum_k\delta(t-t_i^k).\]

Here \(\tau_m>0\) has units of time, and \(\delta(t-t_i^k)\) is a Dirac impulse with inverse-time units \hyperref[ref-gerstner-2014]{{[}15{]}}. The impulse coefficient \(q_i\) therefore has voltage-times-time units so that every term has voltage units. An isolated impulse changes \(V_m\) by \(q_i/\tau_m\). Here \(q_i\) denotes integrated voltage drive. Other synapse models use different weight dimensions according to their equations. Threshold and reset behavior supply additional rules beyond this differential equation.

Between input events and resets, this equation has the closed-form decay \(V_m(t+\Delta)=V_m(t)\exp(-\Delta/\tau_m)\). An event-driven realization can advance to the next event and apply its voltage jump, avoiding fixed-time-step updates across empty intervals. It must also process any threshold or reset event required by the complete model. Event queues and persistent state determine storage, while exponential evaluation and timestamp precision determine the numerical error. The temporal contract therefore exposes both the work that can be skipped and the accuracy required when computation resumes.

For a particular tuple of arriving events, coincidence within a window can be expressed as:

\[\max_i t_i-\min_i t_i\leq\tau.\]

In our PHG design, a hyperedge can retain the event tuple, its joint window, and the firing rule together. For fixed scalar timestamps this inequality is equivalent to requiring \(|t_i-t_j|\leq\tau\) for every pair. The advantage of the hyperedge is the explicit identity of the event group and its consumption policy, especially when several candidate tuples overlap. Our compiler checks the rule\textquotesingle s dimensional and admitted proof obligations. Runtime event processing determines whether a particular tuple meets the window condition.

\subsection*{6.3 Learning from Spike Timing}\label{learning-from-spike-timing}
\addcontentsline{toc}{subsection}{6.3 Learning from Spike Timing}

The temporal model above describes how a network responds to events. Learning can also use that timing information to change how strongly one unit influences another. In a spiking network, a connection is called a synapse, and its weight controls the effect of an incoming spike on the receiving unit. A local learning rule updates that weight using information available at the connection, such as the recent spike histories of the sending and receiving units \hyperref[ref-gerstner-2014]{{[}15{]}}.

Spiking networks can run on conventional processors. Neuromorphic processors implement aspects of this stateful, event-based organization in hardware. Designs such as Loihi keep neuron and connection state close to the processing that uses it and support learning from local event histories \hyperref[ref-davies-2018]{{[}16{]}}. This organization can reduce data movement during adaptation. In our ADM design, it offers a way to place learning alongside event processing, with the update rule and its persistent state included in the target\textquotesingle s resource requirements.

Spike-timing-dependent plasticity (STDP) is a family of such local learning rules. These rules use the relative timing of spikes from the sending (presynaptic) unit and the receiving (postsynaptic) unit to adjust the connection\textquotesingle s weight. For \(\Delta t=t_{\mathrm{post}}-t_{\mathrm{pre}}\) and positive time constants \(\tau_+,\tau_-\), a common pair-based rule has the form:

\[\Delta w_{ij}=\begin{cases}
 A_+\exp(-\Delta t/\tau_+) & \Delta t>0,\\
 -A_-\exp(\Delta t/\tau_-) & \Delta t<0.
\end{cases}\]

The time ratios are dimensionless. The nonnegative amplitudes \(A_+\) and \(A_-\) have the weight\textquotesingle s units. The complete rule also chooses the contributing event pairs, the treatment of coincident events, and bounds on \(w\) \hyperref[ref-gerstner-2014]{{[}15{]}}. Persistent traces can summarize earlier events, so their storage lifetime spans the uses of that history rather than one update\textquotesingle s stack frame.

Forward-gradient methods and local STDP rules both admit resource analysis, but have different information requirements:

{\def\LTcaptype{none} 
\begin{longtable}[]{@{}
  >{\raggedright\arraybackslash}p{(\linewidth - 4\tabcolsep) * \real{0.3333}}
  >{\raggedright\arraybackslash}p{(\linewidth - 4\tabcolsep) * \real{0.3333}}
  >{\raggedright\arraybackslash}p{(\linewidth - 4\tabcolsep) * \real{0.3333}}@{}}
\toprule\noalign{}
\begin{minipage}[b]{\linewidth}\raggedright
Learning method
\end{minipage} & \begin{minipage}[b]{\linewidth}\raggedright
Update information
\end{minipage} & \begin{minipage}[b]{\linewidth}\raggedright
State and lifetime
\end{minipage} \\
\midrule\noalign{}
\endhead
\bottomrule\noalign{}
\endlastfoot
Forward-gradient estimator & Directional derivative of the loss and sampled tangents & Live primal/tangent state and any retained optimizer state \\
Pair-based trace STDP & Local pre/post event history & Persistent traces and synaptic weights \\
\end{longtable}
}

The forward estimator can depend on a model-wide loss signal. The displayed STDP rule uses local event timing, with a broader signal entering only if a modulation term is added. Our shared graph infrastructure records the resource and update requirements of each rule.

On a target with an applicable exact accumulator, a covered sum in either computation can use it. Exponential evaluation and nonlinear weight dynamics retain their own numerical behavior. A neuromorphic learning engine must declare which update rules and state resources it actually supports before the compiler can validate the mapping.

\subsection*{6.4 Hybrid Geometric-Neuromorphic Networks}\label{hybrid-geometric-neuromorphic-networks}
\addcontentsline{toc}{subsection}{6.4 Hybrid Geometric-Neuromorphic Networks}

A hybrid model could use geometric layers for spatial relationships and spiking layers for event dynamics. The conversion between them must state how a continuous value becomes a rate or timing code. Dimensional compatibility is one obligation. Numeric range, temporal resolution, and the information preserved by the encoding are others.

Our PHG would carry these conversion relationships together with candidate target placements. BAREWire, our memory-layout and communication library, supplies the boundary descriptions, including interprocess communication (IPC) or network contracts when needed. The schema establishes how the transmitted fields are interpreted without runtime dimension tags. A checked conversion supplies the associated relationship between source and destination meanings.

Our intended architecture can place geometric computation on an AI Engine and event processing on a neuromorphic target when each destination supports the required operations. The deployed computation combines learned coefficients, persistent temporal state and the conversion between components. Warm rotation must preserve their correspondence when a candidate replaces the active version. Reachability identifies the participating subgraph, while capability, placement and preservation checks determine whether that version can execute on its destination.

\section*{7. Our Versioning and Trust Infrastructure}\label{our-versioning-and-trust-infrastructure}
\addcontentsline{toc}{section}{7. Our Versioning and Trust Infrastructure}

\subsection*{7.1 Evidence Validity Across Updates}\label{evidence-validity-across-updates}
\addcontentsline{toc}{subsection}{7.1 Evidence Validity Across Updates}

Our versioning design associates each model artifact with the obligations discharged for its graph and target. A certificate identifies the premises and result, including the verification procedure and its trusted dependencies. Numerical evidence additionally identifies the matrix construction, sampler, and arithmetic realization it covers. Statistical evidence identifies the model and observation process under which its probability statement holds.

An update\textquotesingle s consequences follow from the facts it changes. Keeping the algebra, basis and support fixed preserves the arithmetic exclusions established by Corollary 5.2. A changed posterior system requires numerical evidence for that system, while a changed temporal rule may alter the required event state. Our version record identifies which results remain applicable and which obligations must be discharged again. A theorem covering a matrix family can support a new instance once membership is checked. An instance-specific residual must be established for the new solve.

The restricted metric follows the same rule. A change of tangent basis within the same subspace permits the explicit transport in §2.4.2. A change of sampler, observation contract or output scale needs its own preservation relation or a new metric calculation. A new parameter point may reuse a proved analytic formula, as in the constant calibration matrix, while a point-specific estimate requires evidence of continued applicability. The PHG retains these dependencies together; DBC supplies their verification obligations and FPS their preservation through compilation.

Native training also carries this contract across foreign calls. An imported sampler or numerical kernel needs the relevant law, effect and arithmetic evidence; an ABI-compatible result alone does not establish the modeled derivative or metric. Repeated calls with equal explicit arguments can depend on different foreign state. FPS develops the boundary preservation account, so these dependencies remain attached to the existing call and evidence identities.

Our PSG joint-constraint mechanism records these relationships. As developed in our companion \emph{Fixed-Point Scaffolding} paper \hyperref[ref-haynes-fps]{{[}11{]}}, a lowering step must preserve the correspondence between established premises and the realized operation. An admitted transformation can transport its evidence under a preservation rule. A change outside that rule creates an obligation to check again. Candidate version records retain these results and their compiler context.

\subsection*{7.1.1 Fresh Randomization After Adaptation}\label{fresh-randomization-after-adaptation}
\addcontentsline{toc}{subsection}{7.1.1 Fresh Randomization After Adaptation}

The bound proved in Urschel\textquotesingle s Theorem 1.2 \hyperref[ref-urschel-2026]{{[}9{]}} (§2.5) applies to every fixed nonsingular matrix under Haar orthogonal preconditioning. We can use this uniformity to derive an adaptive argument. Let \(\mathcal H_t\) contain the history before a new solve, and let that history determine a nonsingular \(n_t\times n_t\) matrix \(A_t\). After fixing \(A_t\), use fresh randomness independent of the history to draw \(Q_t\) from Haar measure on \(O(n_t)\). For a deterministic allocation \(0<\delta_t<1\),

\[
\Pr\!\left[
\operatorname{growth}(Q_tA_t)>
\frac{14n_t^{3/2}\sqrt{\log(en_t)}}{\delta_t}
\,\middle|\,\mathcal H_t\right]\le\delta_t.
\]

Taking expectations and a union bound gives a probability at most \(\sum_t\delta_t\) that any covered growth event occurs. A finite run of \(T\) solves can allocate \(\delta_t=\delta_{\mathrm{run}}/T\). The argument allows the next matrix to depend on prior observations and computations. It does not require a stationary deployment environment for this numerical claim.

The order of selection is a premise. Choosing \(A_t\) after inspecting the current draw, or reusing a revealed draw for a later adaptive matrix, needs another argument. The realized sampler and solve still require the arithmetic bridge in §2.5. Statistical monitoring has its own observation assumptions and risk allocation.

An indefinitely running system can use a summable allocation such as \(\delta_t=\delta_{\mathrm{run}}/[t(t+1)]\), for \(t\ge1\). The associated growth thresholds then increase with \(t\). This bounds aggregate failure probability without establishing a fixed finite carrier range. A deployment must choose a horizon, renewed evidence, or a guarded fallback compatible with its actual resources.

For our pressure model, these choices become part of preparing the next version: which observations changed the posterior, how its new linear system was solved, and which bounds justify using the result. Our version record keeps that evidence with the candidate. A receiving system can then check what was established for the model it is about to run.

\subsection*{7.2 Version Signing and Evidence Integrity}\label{version-signing-and-evidence-integrity}
\addcontentsline{toc}{subsection}{7.2 Version Signing and Evidence Integrity}

A local development history and a signed production record serve different trust requirements. Our version format is intended to support both. Where post-quantum signing is required, the deployment must specify the algorithm and key policy, including how verification, revocation, and time are handled. A signature authenticates the record under that policy and the algorithm\textquotesingle s security assumptions.

The record binds operational evidence and discharged obligations to the candidate configuration. An auditor can trace a reported enclosure back to the matrix, observations, and solver that produced it. Observation validity remains governed by the evidence model in §3.3, and predictive quality by the acceptance procedure.

Our warm-rotation protocol must also bind the installed artifact to the verified version. Authenticating a candidate record is useful only if requests execute the configuration that record identifies. State migration and retained references are part of that correspondence.

\subsection*{7.3 Domain Scope and Composition}\label{domain-scope-and-composition}
\addcontentsline{toc}{subsection}{7.3 Domain Scope and Composition}

A domain interface identifies the quantities and relationships a caller can rely on. In the pressure example, dimensions establish the units, the probability model defines the prediction and its uncertainty, and the numerical enclosure bounds the realized calculation. Domain-specific invariants can refine this interface with further laws. Each returned guarantee refers to those declared quantities and premises.

Our actor design places these contracts at component boundaries. Prospero is the supervisor in this design, with domain computations represented by Olivier actors. The actors\textquotesingle{} resource requirements and message relationships would be recorded in the graph for the ADM integration. BAREWire\textquotesingle s typed descriptions provide memory layout, IPC, and network contracts as appropriate. The endpoints must agree on the schema and conversion semantics before a transfer is admitted.

A conversational language model can use an ADM prediction and its uncertainty as evidence when composing an answer. It can also submit the domain calculations underlying a proposed response or action for evaluation, then revise its output using the returned result. In the pressure example, a request binds the measurements, units, observation window and model version. The returned prediction or threshold decision carries the applicable uncertainty and numerical-error enclosure. The same interface can serve an automated controller. Interpretation and generation must preserve the request\textquotesingle s meaning and the result\textquotesingle s premises, including an unresolved result when the evidence does not support a decision.

For example, a caller may propose the assertion \(D>\tau\) from §3.2. Write the returned interval as \([D_-,D_+]\). Under its declared coverage, \(D_->\tau\) supports the assertion, \(D_+\le\tau\) contradicts it, and an overlapping interval leaves it unresolved. The three rows of the worked table exercise these outcomes. Before generation, the result supplies a quantitative basis for an answer. After a proposed calculation, the same check supplies feedback that can correct an unsupported threshold assertion. A controller can use the supported decision directly. The caller retains the observation assumptions, allocated risk and version identifier with the result, and dependent calls follow §4.1.

An organization could keep operational data within its own trust boundary while using externally obtained models or public information to initialize candidates. Achieving that custodianship depends on deployment configuration and data handling, including what distillation queries disclose. Version provenance helps demonstrate that configuration and its transitions. It complements access controls and the statistical evaluation needed for the domain.

\section*{8. Related Work}\label{related-work}
\addcontentsline{toc}{section}{8. Related Work}

Having followed ADM from model construction through learning and deployment, we now place our proposal alongside related work. These comparisons identify the methods and results we build on, clarify where our design makes a distinct contribution, and guide the evaluation of its precision and efficiency.

\subsection*{8.1 Bayesian Teaching and Prior Construction}\label{bayesian-teaching-and-prior-construction}
\addcontentsline{toc}{subsection}{8.1 Bayesian Teaching and Prior Construction}

Qiu et al. \hyperref[ref-qiu-2026]{{[}14{]}} train language models on examples from a Bayesian assistant and observe improved belief updating in the evaluated recommendation and shopping tasks. Their procedure transfers behavior from an explicit probability model into a language model. Our proposed distillation in §3.4 would obtain information from a general model and construct an explicit prior over domain parameters.

The difference determines what must be evaluated. Success in imitating a Bayesian assistant supplies evidence about the tested behavior. A prior for physical coefficients additionally needs a specified measure and evidence that its uncertainty is useful for the new domain. The matched initialization experiments in §9.1 address that question.

\subsection*{8.2 Geometric Algebra Neural Networks}\label{geometric-algebra-neural-networks}
\addcontentsline{toc}{subsection}{8.2 Geometric Algebra Neural Networks}

Clifford group equivariant networks \hyperref[ref-ruhe-2023]{{[}17{]}} and Clifford-steerable convolutional networks \hyperref[ref-zhdanov-2024]{{[}18{]}} develop geometric architectures with explicit transformation laws. The Flash Clifford implementation \hyperref[ref-zhdanov-2025]{{[}19{]}} addresses efficient realization of geometric operations. These architectures establish structural relationships while using conventional numerical carriers.

Our contribution is to make the relevant dimensions, support, and operation contracts available in the compilation graph. A checked update can preserve those structural requirements, and a lowering witness can relate them to the target realization. The numerical error of that realization remains a distinct property for analysis and measurement.

\subsection*{8.3 Units in Spiking Computation}\label{units-in-spiking-computation}
\addcontentsline{toc}{subsection}{8.3 Units in Spiking Computation}

Brian checks physical units in neural models and reports dimensional mismatches \hyperref[ref-brian-2026]{{[}20{]}}. This provides an established example of dimensional discipline in spiking computation. Our proposed extension combines those kinds of relationships with target resources and cross-component conversion in the same graph used for candidate admission.

A useful comparison would follow the impulse model in §6 through a change of numeric representation and an installation. It would identify where units are checked, how the update rule\textquotesingle s error and storage are justified, and whether that evidence remains valid at the receiving component. Geometric support and temporal behavior retain their respective rules within this shared process.

\subsection*{8.4 Structural Convergence and Numerical Reliability}\label{structural-convergence-and-numerical-reliability}
\addcontentsline{toc}{subsection}{8.4 Structural Convergence and Numerical Reliability}

Sutton\textquotesingle s bitter lesson \hyperref[ref-sutton-2019]{{[}21{]}} argues for methods that benefit from computation at scale, and challenges restrictions that obstruct that benefit. Banko and Brill \hyperref[ref-banko-brill-2001]{{[}22{]}} and Halevy et al. \hyperref[ref-halevy-2009]{{[}23{]}} show the practical value of large data resources. Our proposal must meet the same test: a domain constraint should justify the work it excludes while retaining the behavior the task needs.

We use \emph{structural convergence} for learning within such an established frame. A dimensional relationship can exclude incompatible hypotheses before training while leaving many functions with the same units. Data distinguishes among those functions. A specified probability model expresses the uncertainty, including cases where a surprising observation increases it. The value of the construction depends on its fit to the domain and its complete accuracy/resource tradeoff.

Our original ADM formulation placed numerical precision and resource requirements within the construction of a learning model. Its premise was that known mathematical structure could exclude unnecessary work and support explicit accuracy obligations for the computation that remains. Urschel\textquotesingle s subsequent analysis strengthens the numerical basis of this proposal: exact elimination rules and specified probability laws jointly yield quantitative bounds on factor growth.

The growth bound Urschel proves for elimination without pivoting after Haar orthogonal preconditioning (Theorem 1.2) \hyperref[ref-urschel-2026]{{[}9{]}} is especially relevant to adaptation. It covers any fixed nonsingular system, so a matrix produced by learning need not have independent Gaussian entries. Section 7.1.1 derives the corresponding argument for successive adaptive updates by fixing each matrix before fresh randomization. This gives ADM a way to renew numerical evidence as learned coefficients change.

The engineering consequence is to connect a required accuracy to the resources needed to obtain it. Range, arithmetic-error and sensitivity analysis turn applicable growth bounds into evidence for a realized solve (§2.5). Our posterior and monitoring example then shows how an error enclosure can support a decision or identify when further arithmetic precision cannot resolve it. Our architectural contribution is to carry this evidence through model updates, installation and the calling interface. Section 9 specifies how to evaluate the resulting prediction quality and total resource cost, including the work required to establish and retain the numerical guarantee.

\subsection*{8.5 Mixture of Experts, Agentic Frameworks, and Actors}\label{mixture-of-experts-agentic-frameworks-and-actors}
\addcontentsline{toc}{subsection}{8.5 Mixture of Experts, Agentic Frameworks, and Actors}

Sparse Mixture of Experts \hyperref[ref-shazeer-2017]{{[}24{]}} uses learned routing to select computation within a model. Actor composition describes separately managed components and their communication. These mechanisms can be combined. The number of free parameters in an expert\textquotesingle s admissible space may inform its training cost, while estimator variance and routing quality still require the relevant analysis.

Our Fidelity actor design would give each domain component explicit message and resource contracts. BAREWire descriptions then connect those contracts to the deployed endpoints. Agentic and tool-use systems can also employ typed APIs and schemas. The comparison should follow a shared contract through candidate preparation, boundary conversion, and installed execution, identifying which relationships each system establishes and preserves.

Section 7.3 gives the corresponding ADM interaction. The comparison with other tool-use designs concerns the evidence returned with a calculation and its validity across dependent calls and model replacement, as well as the interface schema.

\section*{9. Future Work}\label{future-work}
\addcontentsline{toc}{section}{9. Future Work}

We propose comparing prediction error at matched resource budgets and resource use at matched prediction quality. Prior initialization, numerical computation and deployment each contribute to that tradeoff. The experiments would separate their effects and measure the combined cost.

\subsection*{9.1 Bayesian Distillation: Construction and Evaluation}\label{bayesian-distillation-construction-and-evaluation}
\addcontentsline{toc}{subsection}{9.1 Bayesian Distillation: Construction and Evaluation}

Our first distillation experiment must implement an observation interface and one of the prior constructions in §3.4, with the likelihood used for subsequent evidence. These choices precede any claim about posterior calibration or computational cost.

Evaluation should compare initialization methods under matched model and data budgets. Relevant questions include how much domain information a source model contributes, how extraction queries affect the result, and whether a smaller specialized source is more useful than a larger general one. Structural checks establish compatibility, while predictive and calibration measures establish statistical quality.

\subsection*{9.2 Numerical and Monitoring Evaluation}\label{numerical-and-monitoring-evaluation}
\addcontentsline{toc}{subsection}{9.2 Numerical and Monitoring Evaluation}

Our first numerical experiment should implement the posterior and monitor in §3 with a high-precision reference, then vary conditioning and matrix size. Compare the appropriate native Cholesky or square-root solve with any admitted randomized method on the same target problem. Record factor and intermediate magnitudes, realized pivots, and enclosed residuals. Check mean and covariance errors separately, including covariance positive definiteness. Sampling and transform formation belong in the measured time and peak memory. Empirical tails at tested sizes remain empirical evidence, distinct from a theorem\textquotesingle s failure bound.

Continue each calculation through clearly above-threshold, below-threshold, and overlapping decision intervals. Measure false alarms, missed triggers, and undecided outcomes under the declared observation process. Include calibration and the cost of extending a window or invoking a fallback. Repeated and adaptive use should follow its allocated risk policy. Cases with correlated observations or changed variance test whether an implementation detects an unmet monitoring premise rather than silently applying the illustrative Gaussian formula.

For the tangent path, implement the calibration example in §2.4.3 as an analytic reference, then compare exact and execution-estimated restricted metrics. Include the full tangent basis, a basis change spanning the same subspace, a subspace that misses the offset, and singular matrices with compatible and incompatible loss covectors. Deliberately substitute \(U^{\mathsf T}U\) for \(K\), omit metric-estimation error from the enclosure, and reuse evidence after changing an output scale or a trace presentation. Each case must either preserve the stated update or expose the unmet premise. A discarded parameter-dependent random draw distinguishes preservation of the output law from preservation of trace geometry.

Compare Euclidean and metric updates under declared step-size policies and matched total work, including trace sampling, velocity pairings, factorization and checking. Measure finite-step loss and prediction quality separately from the first-order descent result. A gradient-accuracy contract must identify its metric and the source of its residual bound. Reaching a work or memory limit with that bound unresolved needs a declared stopping or fallback policy. Curvature and posterior approximations require the additional information identified in §3.1.

For conversational use, compare the same language model with and without ADM calls while keeping the user interface fixed. Measure violations of domain constraints and numerical errors across turns, faithful use of returned enclosures, handling of unresolved calls, and total cost. Include cases where a response misstates units, omits uncertainty, or proposes a calculation outside the ADM\textquotesingle s declared domain. This evaluates the proposed reduction in conversational domain errors separately from the distribution-shift monitor in §3.2.

\subsection*{9.3 Resource and Installation Evidence}\label{resource-and-installation-evidence}
\addcontentsline{toc}{subsection}{9.3 Resource and Installation Evidence}

A working model should measure peak live storage during concurrent inference and candidate preparation. Include solver and checking workspaces alongside tangent and optimizer state. Quire accumulation requires a partial-sum invariant for its declared capacity. A rotor-bearing layer needs its transformation law and numerical error bound. Compare complete configurations at matched prediction quality, with training time and inference latency reported alongside memory and transfer cost.

The comparisons should separate the sources of savings. For the bivector product, compare a full expansion with masked inputs against an implementation that uses the declared support, and report both product counts and realized time. For posterior queries, compare full-covariance construction with factor reuse for the requested predictive variance. For training, report the admissible parameter count, tangent strategy and live state for each configuration. Reaching an error tolerance earlier is useful when the combined learning, inference and checking cost also meets the deployment objective. Account for the full work required to reach acceptance, including precision tuning, diagnostic recomputation, rejected candidates and repeated training runs. When isolating arithmetic effects, use matched initial states, data, update rules, sampled directions and acceptance criteria.

Candidate admission needs discriminating negative cases: a known uncovered carrier range, adequate range with inadequate precision, dependent tangents, and a stale matrix-to-evidence binding. Changing the rounding rule or sampler should invalidate evidence that depends on it. For an adaptive Haar route, deliberately reveal or reuse the draw before matrix selection and verify that the fresh-draw theorem instance is no longer accepted.

Installation checks should introduce mismatches between candidate identity, accepted evidence, and operational state. Requests must retain a coherent version while old references are released or migrated under the selected protocol. Measure replacement latency and peak concurrent storage, including buffering. These experiments would establish whether the structural and numerical evidence survives the transition that warm rotation is intended to provide.

\section*{10. Adaptive Continuity}\label{adaptive-continuity}
\addcontentsline{toc}{section}{10. Adaptive Continuity}

We develop ADMs around the joint requirement for precision and efficiency. Declared support excludes parameter coordinates and arithmetic paths before learning. Forward differentiation supplies directional measurements, and a declared program metric turns those measurements into a restricted natural-gradient calculation. Its error account separates uncertainty in that metric from derivative and solve errors. Our worked constructions give concrete consequences: a calibration update with a known metric, nine coefficient products for the specified bivector multiplication, a reused factor for predictive uncertainty, and a monitoring enclosure that identifies when greater arithmetic precision cannot settle the decision.

Those results have the same meaning whether consumed by an automated controller or through a language model. Our calling interface retains the quantities, assumptions and uncertainty needed to use them. Warm rotation binds those results to the admitted model version as learning continues, including the dependent calls used to prepare and check a response. This is the contribution of our combined architecture: efficient domain computations whose numerical evidence remains available to the systems that use them.

The next implementation should carry the worked example from observation through installation and both calling pathways on a declared target. Section 9 specifies the comparisons needed to measure prediction quality, conversational domain errors and total resource cost. Those measurements test how much the established reductions and error controls contribute in a deployed system.

\section*{Acknowledgments}\label{acknowledgments}
\addcontentsline{toc}{section}{Acknowledgments}

The author thanks John L. Gustafson for guidance on posit arithmetic, the b-posit bounded-regime arithmetic pathway, and for \emph{Every Bit Counts: Posit Computing} \hyperref[ref-gustafson-2024]{{[}25{]}}, whose Chapter 13 provides the definitive parameterization reference for application-specific posit systems. The treatment of representation selection for training arithmetic in Sections 2 and 4 reflects his technical influence directly. The author thanks Barak A. Pearlmutter for his encouragement of this research program and for his foundational work on automatic differentiation, which, together with the forward gradient result of Baydin et al. \hyperref[ref-baydin-2022]{{[}5{]}}, underpins the training substrate described in this paper.

The author also thanks Don Syme for his influence on Clef\textquotesingle s language design. His co-authorship of the forward gradient paper connected this work to the broader AD research community. The author also thanks Paul Snively for contributions on delimited continuations and the geometric algebra problem framing documented in the companion PHG paper \hyperref[ref-haynes-phg]{{[}4{]}} that informed this work.

Martín Coll\textquotesingle s work on interaction-net rewriting in MLIR \hyperref[ref-coll-2025]{{[}26{]}} informed the investigation of joint reduction in our PSG/PHG. The author thanks him also for his contributions to the hypergraph interchange format \hyperref[ref-coll-hif-2025]{{[}27{]}}, a reference for interchange of higher-order network structure.

\section*{Software Availability}\label{software-availability}
\addcontentsline{toc}{section}{Software Availability}

We develop the Clef language, Composer compiler, and supporting libraries within our Fidelity Framework project. Source repositories are available at \url{https://github.com/FidelityFramework}. The language specification, compiler documentation, and ongoing project research are available at \url{https://clef-lang.com}. Licensing terms are detailed in each repository. The repositories contain compiler and BAREWire implementations alongside design material. Our ADM training, numerical-evidence, and warm-rotation integrations remain work to implement and evaluate.

\section*{References}\label{references}
\addcontentsline{toc}{section}{References}

\leavevmode\protect\phantomsection\label{ref-kennedy-2009}{}{[}1{]} Kennedy, A. (2009). Types for units-of-measure: Theory and practice. In \emph{Central European Functional Programming School}, LNCS 6299. Springer.

\leavevmode\protect\phantomsection\label{ref-haynes-dts-dmm}{}{[}2{]} Haynes, H. (2026). Dimensional type systems and deterministic memory management: Design-time semantic preservation in native compilation. \href{https://arxiv.org/abs/2603.16437}{arXiv:2603.16437}.

\leavevmode\protect\phantomsection\label{ref-petricek-2014}{}{[}3{]} Petricek, T., Orchard, D., and Mycroft, A. (2014). Coeffects: A calculus of context-dependent computation. In \emph{Proceedings of ICFP}.

\leavevmode\protect\phantomsection\label{ref-haynes-phg}{}{[}4{]} Haynes, H. (2026). The program hypergraph: Multi-way relational structure for geometric algebra, spatial compute, and physics-aware compilation. \href{https://arxiv.org/abs/2603.17627}{arXiv:2603.17627}.

\leavevmode\protect\phantomsection\label{ref-baydin-2022}{}{[}5{]} Baydin, A. G., Pearlmutter, B. A., Syme, D., Wood, F., and Torr, P. (2022). Gradients without backpropagation. \emph{arXiv preprint arXiv:2202.08587}.

\leavevmode\protect\phantomsection\label{ref-flugel-2026}{}{[}6{]} Flügel, K., Coquelin, D., Weiel, M., Debus, C., Streit, A., and Götz, M. (2026). Beyond backpropagation: Optimization with multi-tangent forward gradients. \href{https://arxiv.org/abs/2410.17764v2}{arXiv:2410.17764v2}. IJCNN 2025; revised preprint January 2026.

\leavevmode\protect\phantomsection\label{ref-karwowski-2026}{}{[}7{]} Karwowski, J., Nielsen, F., and Staton, S. (2026). Compositional information geometry of probabilistic programs. \href{https://arxiv.org/abs/2610.17886v1}{arXiv:2610.17886v1}, October 13, 2026.

\leavevmode\protect\phantomsection\label{ref-jonnalagadda-2025}{}{[}8{]} Jonnalagadda, A. A., Thotli, R., and Gustafson, J. L. (2025). Closing the gap between float and posit hardware efficiency. In \emph{Conference on Next Generation Arithmetic}. arXiv preprint arXiv:2603.01615.

\leavevmode\protect\phantomsection\label{ref-urschel-2026}{}{[}9{]} Urschel, J. (2026). On the growth factor of random matrices. \href{https://arxiv.org/abs/2610.06785v1}{arXiv:2610.06785v1}, October 5, 2026.

\leavevmode\protect\phantomsection\label{ref-haynes-dbc}{}{[}10{]} Haynes, H. (2026). Decidable by construction: Design-time verification for trustworthy AI. \href{https://arxiv.org/abs/2603.25414}{arXiv:2603.25414}.

\leavevmode\protect\phantomsection\label{ref-haynes-fps}{}{[}11{]} Haynes, H. (2026). Fixed-point scaffolding in the Clef programming language. \href{https://arxiv.org/abs/2606.02854}{arXiv:2606.02854}.

\leavevmode\protect\phantomsection\label{ref-lattner-2021}{}{[}12{]} Lattner, C., et al. (2021). MLIR: Scaling compiler infrastructure for domain specific computation. In \emph{Proceedings of CGO}.

\leavevmode\protect\phantomsection\label{ref-haynes-nft}{}{[}13{]} Haynes, H. (2026). Negative and fractional types in the Fidelity Framework. \href{https://arxiv.org/abs/2606.04352}{arXiv:2606.04352}.

\leavevmode\protect\phantomsection\label{ref-qiu-2026}{}{[}14{]} Qiu, L., Sha, F., Allen, K., Kim, Y., Linzen, T., and van Steenkiste, S. (2026). Bayesian teaching enables probabilistic reasoning in large language models. \emph{Nature Communications}, 17, 1238. \url{https://doi.org/10.1038/s41467-025-67998-6}.

\leavevmode\protect\phantomsection\label{ref-gerstner-2014}{}{[}15{]} Gerstner, W., Kistler, W. M., Naud, R., and Paninski, L. (2014). \emph{Neuronal Dynamics: From Single Neurons to Networks and Models of Cognition}. Cambridge University Press. \href{https://neuronaldynamics.epfl.ch/online/index.html}{Online edition}.

\leavevmode\protect\phantomsection\label{ref-davies-2018}{}{[}16{]} Davies, M., et al. (2018). Loihi: A neuromorphic manycore processor with on-chip learning. \emph{IEEE Micro}, pp. 82--99. \href{https://redwood.berkeley.edu/wp-content/uploads/2021/08/Davies2018.pdf}{Paper}.

\leavevmode\protect\phantomsection\label{ref-ruhe-2023}{}{[}17{]} Ruhe, D., Brandstetter, J., and Forré, P. (2023). Clifford group equivariant neural networks. \emph{arXiv preprint arXiv:2305.11141}.

\leavevmode\protect\phantomsection\label{ref-zhdanov-2024}{}{[}18{]} Zhdanov, M., et al. (2024). Clifford-steerable convolutional neural networks. In \emph{Proceedings of ICML}.

\leavevmode\protect\phantomsection\label{ref-zhdanov-2025}{}{[}19{]} Zhdanov, M. (2025). Flash Clifford: Hardware-efficient implementation of Clifford algebra neural networks. \href{https://github.com/maxxxzdn/flash-clifford}{Repository}.

\leavevmode\protect\phantomsection\label{ref-brian-2026}{}{[}20{]} Brian contributors. (2026). Physical units. \emph{Brian 2 documentation}. \href{https://brian2.readthedocs.io/en/stable/user/units.html}{Units guide}. Accessed October 7, 2026.

\leavevmode\protect\phantomsection\label{ref-sutton-2019}{}{[}21{]} Sutton, R. S. (2019). The bitter lesson. Incomplete Ideas blog, March 2019. \href{http://incompleteideas.net/IncIdeas/BitterLesson.html}{Essay}.

\leavevmode\protect\phantomsection\label{ref-banko-brill-2001}{}{[}22{]} Banko, M., and Brill, E. (2001). Scaling to very very large corpora for natural language disambiguation. In \emph{Proceedings of ACL}.

\leavevmode\protect\phantomsection\label{ref-halevy-2009}{}{[}23{]} Halevy, A., Norvig, P., and Pereira, F. (2009). The unreasonable effectiveness of data. \emph{IEEE Intelligent Systems}, 24(2):8-12.

\leavevmode\protect\phantomsection\label{ref-shazeer-2017}{}{[}24{]} Shazeer, N., et al. (2017). Outrageously large neural networks: The sparsely-gated mixture-of-experts layer. \emph{arXiv preprint arXiv:1701.06538}.

\leavevmode\protect\phantomsection\label{ref-gustafson-2024}{}{[}25{]} Gustafson, J. L. (2024). \emph{Every Bit Counts: Posit Computing}. Chapman and Hall/CRC Computational Science. CRC Press, Boca Raton, FL. ISBN 978-1-032-73805-5.

\leavevmode\protect\phantomsection\label{ref-coll-2025}{}{[}26{]} Coll, M. (2025). Inet dialect: Declarative rewrite rules for interaction nets. MLIR Open Design Meeting, April 2025.

\leavevmode\protect\phantomsection\label{ref-coll-hif-2025}{}{[}27{]} Coll, M., Joslyn, C. A., Landry, N. W., Lotito, Q. F., Myers, A., Pickard, J., Praggastis, B., and Szufel, P. (2025). HIF: The hypergraph interchange format for higher-order networks. \emph{arXiv preprint arXiv:2507.11520}.
\end{document}